%% file: main.tex
\documentclass[onefignum,onetabnum]{siuro210301}

\usepackage{lipsum}
\usepackage{amsfonts}
\usepackage{graphicx}
\usepackage{epstopdf}
\usepackage{algorithmic}
\ifpdf
  \DeclareGraphicsExtensions{.eps,.pdf,.png,.jpg}
\else
  \DeclareGraphicsExtensions{.eps}
\fi

\usepackage{enumitem}
\setlist[enumerate]{leftmargin=.5in}
\setlist[itemize]{leftmargin=.5in}

\newsiamremark{remark}{Remark}
\newsiamremark{hypothesis}{Hypothesis}
\crefname{hypothesis}{Hypothesis}{Hypotheses}
\newsiamthm{claim}{Claim}

\input{macros}

\headers{A Tensor SVD-based Classification Algorithm Applied to fMRI Data}{K. Keegan, T. Vishwanath, and Y. Xu}

\title{A Tensor SVD-based Classification Algorithm Applied to fMRI Data}

\author{Katherine Keegan\thanks{Department of Mathematics, Mary Baldwin University, Staunton, VA
  (\email{keegank1624@marybaldwin.edu}).}
\and Tanvi Vishwanath\thanks{Department of Mathematics, Texas A \& M University, College Station, TX
  (\email{tanvivishwanath@tamu.edu}).}
\and Yihua Xu\thanks{Department of Mathematics, Georgia Institute of Technology, Atlanta, GA 
  (\email{yxu604@gatech.edu}).}}

\dedication{\small\textit{Project advisors: Dr.~Elizabeth Newman\thanks{Department of Mathematics, Emory University, Atlanta, GA (\email{elizabeth.newman@emory.edu).}}, Vida John\thanks{Tesselations School, Cupertino, CA}}}

\usepackage{amsopn}

\ifpdf
\hypersetup{
  pdftitle={Tensor-Based Approaches to fMRI Classification},
  pdfauthor={K. Keegan, T. Vishwanath, Y. Xu}
}
\fi

\begin{document}

\maketitle

\begin{abstract}
To analyze the abundance of multidimensional data, tensor-based frameworks have been developed. Traditionally, the matrix singular value decomposition (SVD) is  used to extract the most dominant features from a matrix containing the vectorized data. While the SVD is highly useful for data that can be appropriately represented as a matrix, this step of vectorization causes us to lose the high-dimensional relationships intrinsic to the data. To facilitate efficient multidimensional feature extraction, we utilize a projection-based classification algorithm using the t-SVDM, a tensor analog of the matrix SVD. Our work extends the t-SVDM framework and the classification algorithm, both initially proposed for tensors of order 3, to any number of dimensions. We then apply this algorithm to a classification task using the StarPlus fMRI dataset. Our numerical experiments demonstrate that there exists a superior tensor-based approach to fMRI classification than the best possible equivalent matrix-based approach. Our results illustrate the advantages of our chosen tensor framework, provide insight into beneficial choices of parameters, and could be further developed for classification of more complex imaging data. We provide our Python implementation at~\url{https://github.com/elizabethnewman/tensor-fmri}. 
\end{abstract}


\section{Introduction}
\label{sec:introduction}
High-dimensional data  has become increasingly prevalent in various fields such as video, semantic indexing, and  chemistry \cite{KoldaBader2009:tensorBackground}. One such area is medical imaging, in which the high-dimensional data of interest are medical images obtained using a modality such as functional Magnetic Resonance Imaging (fMRI). A single patient's fMRI scan can be thought of as a series of three-dimensional scans over time \cite{ADOLF20113760,CHATZICHRISTOS201917}.

fMRI data can serve an important role in disease detection, with applications ranging from Alzheimer’s to depression~\cite{fmri-information,SpatioTemporalTensorKernel}. Human medical professionals can be trained to interpret fMRI scans and make diagnostic decisions based on visual inspection. Various researchers have investigated the viability of utilizing mathematical methods to automate this classification process. The key to classification algorithms is to extract meaningful features from training data belonging to a certain class. The Singular Value Decomposition (SVD) is one popular method that accomplishes this task successfully in various data analysis contexts. However, it requires that all of the input data is reshaped into a two-dimensional form.

A tensor is a higher-dimensional analog to a matrix. Unlike a matrix in traditional linear algebra, where all information is indexed according to two possible axes (rows and columns), a tensor can have arbitrarily many dimensions. The mathematical foundations of algebra using tensors have been well-developed; \cite{KoldaBader2009:tensorBackground} provides some useful background on this subject. In order to identify dominant features from inherently multidimensional data such as fMRI scans, tensor-based SVD analogs have been developed that allow the data to retain its multidimensional form throughout the classification process, avoiding the pitfalls that occur with vectorizing such data \cite{KernfeldKilmerAeron2015}.

We focus on the t-SVDM framework, a tensor SVD approach proposed in \cite{kilmer2019tensortensor}. This particular framework not only allows for a kind of multiplication between tensors that appears similar to traditional matrix multiplication, but also can provide representations of multidimensional data that are provably better than representing the data in vectorized form as a matrix. In order to determine its viability in more complicated medical diagnostic classification tasks, we apply this framework to a classification task of determining whether or not a human subject is reading a sentence or viewing a picture.  We illustrate that this t-SVDM classification algorithm successfully beats its matrix counterpart. Our results also shed some light on the potential limitations of this method.

In our paper, we illustrate an extension of the t-SVDM framework, describe how it can be used for a classification algorithm based on \cite{8313137}, and we apply this algorithm to the aforementioned fMRI classification task. The main extensions and contributions of our work are the following:
\begin{itemize}
    \item \textbf{Dimension Extension:} The original t-SVDM framework is proposed for only three-dimensions. Our paper provides definitions for a $p$-dimensional tensor and illustrates the usability of this framework for a five-dimensional fMRI dataset.
    \item \textbf{Transformation Choices:} We select the t-SVDM not only for its ability to process high-dimensional data, but also for the flexibility that the framework introduces via the $\starM$-product, which enables one to strategically choose a mathematical transformation based on the nature of the data being analyzed.
    \item \textbf{Algorithm Flexibility:} The t-SVDM and our proposed classification procedure is a mathematically justified framework that can be applied to any labeled high-dimensional data. Thus, all of the methods described in our paper can easily be extended to other similar classification tasks with labeled data. 
    \item \textbf{Region of Interest Implications:} When incorporating knowledge about specific regions of the brain into our classification procedure, we discovered that the most impactful regions vary between human subjects. This could illustrate the anatomical differences between humans when completing cognitive tasks and highlights the challenge of constructing  the difficulty of creating a universal basis that represents all subjects with consistent accuracy.
\end{itemize}

The paper is organized as follows: \Cref{sec:background} describes necessary background notations, the $\starM$-product and the algebraic framework it gives rise to, and the t-SVDM.  \Cref{sec:local} describes the local tensor SVD approaches specifically for classification and provides an intuition example to provide more explanation. \Cref{sec:numericalExperiments} first describes specific choices of transformation matrices, and then examines numerical results from applying the classification algorithm to the StarPlus fMRI dataset. Finally, \Cref{sec:conclusions} concludes the paper and proposes some potential future directions.

\section{Background and Preliminaries}
\label{sec:background}

In this paper, a \emph{tensor}, denoted with a capital caligraphic letter $\Acal$, is a multidimensional array. 
The \emph{order} of a tensor is the number of dimensions it has; if $\Acal$ is a tensor of order-$p$, then the size of $\Acal$ is $n_1\times n_2\times \dots \times n_p$.  
We assume tensors are real-valued; that is, $\Acal\in \Rbb^{n_1\times n_2\times \cdots \times n_p}$.
A \emph{matrix} (order-$2$ tensor) is denoted with a bold capital letter, $\bfA\in \Rbb^{n_1\times n_2}$, a \emph{vector} (order-$1$ tensor) is denoted with a bold lowercase letter, $\bfa\in \Rbb^{n_1}$, and a \emph{scalar} (order-$0$ tensor) is denoted by a lowercase letter, $a\in \Rbb$.
To provide relevant information needed to understand the t-SVDM framework, we will first denote the basic notation and definitions used within this paper. Next, we will introduce tensor products including $\starM$-product. This notation and terminology will then equip us with the tools needed to explain the t-SVDM and the important properties it offers for classification.
For simplicity, we consider real-valued tensors, but the definitions and theory presented can be extended to complex-valued tensors as well.
\subsection{Notation}
\label{sec:notation}

Consider that an $m \times n$ matrix is structured with $m$ rows and $n$ columns.  Using {\sc Matlab} notation, $\bfA(i,:)$ denotes the $i$-th row and $\bfA(:,j)$ denotes the $j$-th column.  We can similarly describe the structure of a tensor. These analogs of matrix concepts to tensors are introduced and developed in  \cite{KilmerMartin2011} and \cite{KernfeldKilmerAeron2015}. We extend these definitions such that they describe any finite-dimensional tensor of order $p$.

\begin{definition}[mode-$k$ fibers]
Fibers are sections of a tensor $\Acal$ such that all but the $k$-th dimension are fixed. 
\end{definition}

\begin{figure}[h]
    \centering
    \begin{tabular}{cccc}
    \multicolumn{4}{c}{\includegraphics[width=0.18\textwidth]{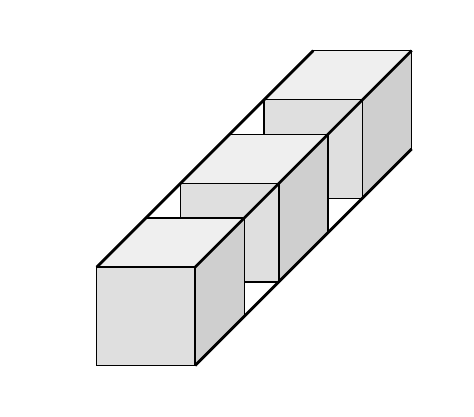}}\\
    \multicolumn{4}{c}{tensor}\\
    \multicolumn{4}{c}{$\Acal$}\\
     \includegraphics[width=0.18\textwidth]{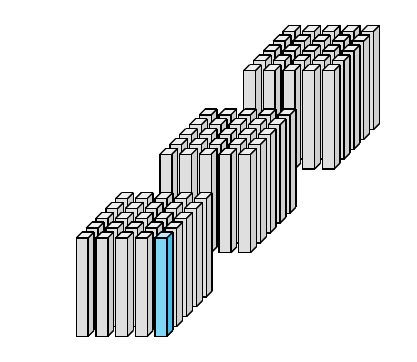}
    & \includegraphics[width=0.18\textwidth]{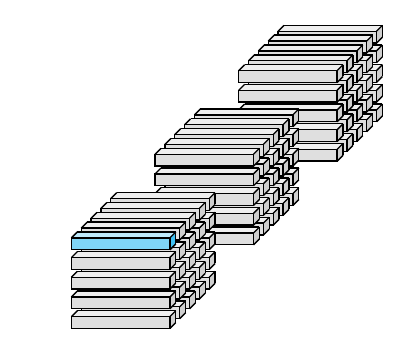}
    & \includegraphics[width=0.18\textwidth]{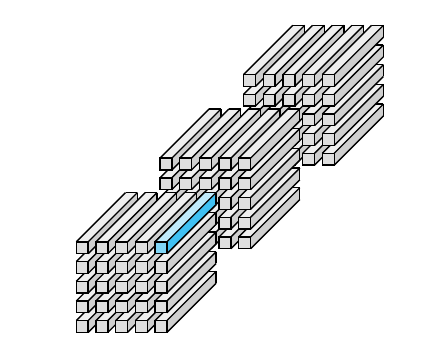}
    & \includegraphics[width=0.18\textwidth]{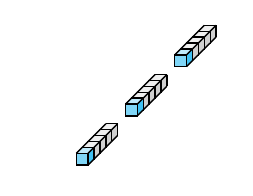}\\
     mode-$1$ fibers & mode-$2$ fibers & mode-$3$ fibers & mode-$4$ fibers \\
     $\Acal(:,i_2,i_3,i_4)$ & $\Acal(i_1,:,i_3,i_4)$ & $\Acal(i_1,i_2,:,i_4)$ & $\Acal(i_1,i_2,i_3,:)$\\
      $\Acal_{:,i_2,i_3,i_4}$ & $\Acal_{i_1,:,i_3,i_4}$ & $\Acal_{i_1,i_2,:,i_4}$ & $\Acal_{i_1,i_2,i_3,:}$\\
    \end{tabular}
    \caption{Visualization of a 4D tensor, its mode-$1$, $2$, $3$, and $4$ fibers. 
    For illustrative purposes, we only show a mode-$4$ fiber for $1\times 1\times n_3\times n_4$ tensor.}
    \label{fig:tensorNotation4D-fiber}
\end{figure}

\begin{definition}[lateral slices, frontal slices, and tubes]\label{def:slices}
For a $p$-dimensional tensor $\Acal$, we denote lateral slices as $\vec{\Acal}_{i_2} = \Acal(:,i_2,:,\dots,:)$, and frontal slices as $\Acal(:,:,i_3,i_4,\dots,i_p)$. 
We denote tubes, or $1\times 1\times n_3\times \dots \times n_p$ tensors, as $\bfa_{i_1,i_2} = \Acal(i_1,i_2,:,\dots,:)$.
\end{definition}

\Cref{fig:tensorNotation4D} visualizes a tensor of order 4, its lateral and frontal slices, and tubes. 

\begin{figure}[h]
    \centering
    \begin{tabular}{ccc}
    \includegraphics[width=0.18\textwidth]{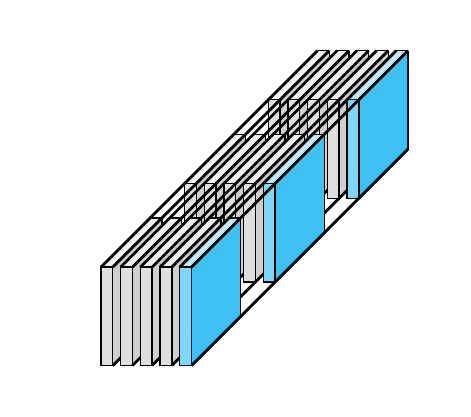}
    & \includegraphics[width=0.18\textwidth]{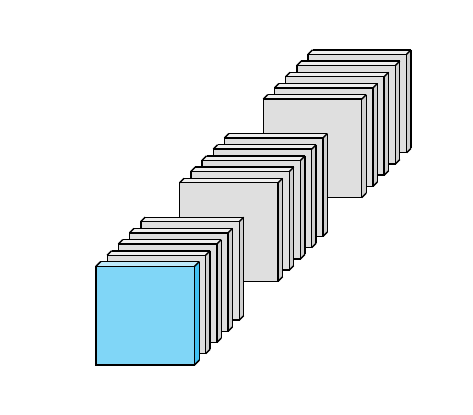}
    & \includegraphics[width=0.18\textwidth]{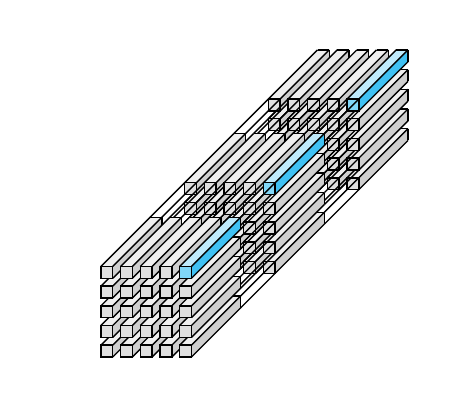}\\
    lateral slices & frontal slices & tubes \\ 
    $\vec{\Acal}_{i_2} = \Acal(:,i_2,:,:)$  & $\Acal(:,:,i_3,i_4)$  & $\Acal(i_1,i_2,:,:)$ \\
    $\Acal_{:,i_2,\dots}$  & $\Acal_{:,:,i_3,i_4}$  & $\bfa_{i_1,i_2}$ \\
    \end{tabular}
    \caption{Visualization of a 4D tensor, its lateral slices, frontal slices, and tubes. Notice that lateral slices fix the second dimension, frontal slices fix all dimensions except the first two, and tubes fix the first two dimensions.}
    \label{fig:tensorNotation4D}
\end{figure}

\begin{definition}[vectorize]\label{def:vectorize}
Matrix vectorization is the process of converting a matrix of values into a column vector. For example, for a matrix $\bfA \in \mathbb{R}^{n_1 \times n_2}$, we have
\begin{align*}
    \myVec(\bfA) &= \myVec\left(\begin{bmatrix} \bfA_{:,1} & \cdots & \bfA_{:,n_2} \end{bmatrix}\right) = \begin{bmatrix}\bfA_{:,1}\\ \vdots \\ \bfA_{:,n_2} \end{bmatrix}
\end{align*}
Tensor vectorization follows a similar recursive pattern.
\begin{align*}
    \myVec(\Acal) &= \begin{bmatrix} \myVec(\Acal_{:,:,1}) \\ \vdots \\ \myVec(\Acal_{:,:,n_3})\end{bmatrix}\text{, where }\Acal \in \Rbb^{n_1\times n_2 \times n_3}\\
    \myVec(\Acal) &= \begin{bmatrix} \myVec(\Acal_{:,\dots,:,1}) \\ \vdots \\ \myVec(\Acal_{:,\dots,:,n_p})\end{bmatrix}\text{, where }\Acal\in \Rbb^{n_1\times n_2\times \cdots \times n_p}
\end{align*}

\end{definition}

Just as multiplication can be accomplished between matrices, one can also define multiplication between a tensor with a matrix. \Cref{def:modekUnfold} and \Cref{def:modekProduct} describe this process. 

\begin{definition}[mode-$k$ unfolding/folding]\label{def:modekUnfold}
A mode-$k$ unfolding of a tensor $\Acal\in\mathbb{R}^{n_1\times\dots\times n_p}$ results in a matrix denoted by $\Acal_{(k)} \in \mathbb{R}^{n_{k} \times (n_{1} \dots n_{k-1} n_{k+1}\dots n_{p})}$ such that the mode-$k$ fibers are the columns of the resultant matrix. A mode-$k$ folding is the reverse of this process. 
\end{definition}

\begin{definition}[mode-$k$ product]\label{def:modekProduct}
The mode-k product of a tensor $\Acal\in\mathbb{R}^{n_1\times\dots\times n_p}$ with a matrix $\bfM\in\mathbb{R}^{d \times n_k}$ results in a tensor whose mode-k unfolding is $\bfM$ multiplied with the mode-k unfolding of $\Acal$, such that:
\[
\Acal\times_k\bfM=\textup{fold}(\bfM\Acal_{(k)}),
\]
with $\textup{fold}(\bfM\Acal_{(k)}) \in \mathbb{R}^{n_{1} \times \dots \times n_{k-1} \times d \times n_{k+1} \times n_{p}}$.

\end{definition}

\begin{definition}[Frobenius norm]\label{def:frobenius}
The Frobenius norm of an order-$p$ tensor $\Acal\in\mathbb{R}^{n_1\times\dots\times n_p}$ is given by $\lVert \Acal  \rVert _{F}= \sqrt{\sum_{i=1_{1}}^{n_{1}} \sum_{i_{2}=1}^{n_{2}} \dots \sum_{i_{p}=1}^{n_{p}} |\Acal_{i_{1},i_{2},...,i_{p}}|^2}$.

\end{definition}

\begin{definition}[$f$-diagonal]
\label{def:fdiagonal}
A tensor is said to be facewise-diagonal, or $f$-diagonal, if its entries only lie along the diagonal of its frontal slices.
\end{definition}

\subsection{Tensor-Tensor Products}
\label{sec:tensorTensorProducts}

In this section, we introduce the $\starM$-product and describe its flexibility, which gives rise to tensor analogies to various familiar matrix concepts such as the transpose, the identity, and orthogonality under $\starM$-product. Kilmer and Martin  originated the method of multiplying matrices for the t-product in \cite{KilmerMartin2011}, and  Kernfeld et al. extended this to the $\starM$-product in \cite{KernfeldKilmerAeron2015}.  We describe the key definitions from these  works for order-$p$ tensors.
\begin{definition}[Facewise Product]
\label{def:fprod}
The facewise product multiplies each of the frontal slices of two tensors in the transform domain in parallel. Given $\Acal\in\mathbb{R}^{n_1\times m\times n_3\times\dots\times n_p}$, and $\Bcal\in\mathbb{R}^{m\times n_2\times n_3\times\dots\times n_p}$, the facewise product of $\Acal$ and $\Bcal$, denoted using ``$\triangle$" can be written as follows: 
\[
    \Ccal = \Acal\triangle\Bcal
\]
where for each frontal slice of $\Ccal$
\[
\Ccal(:,:,i_3,i_4,\dots,i_p)  = \Acal(:,:,i_3,i_4,\dots,i_p) \cdot \Bcal(:,:,i_3,i_4,\dots,i_p), 
\]
for $i_k = 1,\dots,n_k$, where $k = 3,\dots,p$. 
\end{definition}

Now that we have defined both the mode-$k$ product and the facewise product, we can introduce the $\starM$-product. \Cref{def:starm} defines the $\starM$-product, \Cref{alg:starMProd} demonstrates its computation, \Cref{fig:starM_visualization} demonstrates the product's computation for fourth-order tensors, and \Cref{exam:starm} illustrates a simple third-order example of computing the $\starM$-product.

\begin{definition}[$\starM$-product]\label{def:starm}
Given $\Acal\in\mathbb{R}^{n_1\times m\times n_3\times\dots\times n_p}$, and $\Bcal\in\mathbb{R}^{m\times n_2\times n_3\times\dots\times n_p}$, with invertible matrices $\bfM_3\in\mathbb{R}^{n_3\times n_3},\dots,\bfM_p\in\mathbb{R}^{n_p\times n_p}$, define $\Ccal\in\mathbb{R}^{n_1\times n_2\times\dots\times n_p}$ to be the $\starM$-product of $\Acal$ and $\Bcal$ such that 
    \[
    \Ccal=\Acal\starM\Bcal=(\hat{\Acal}\triangle\hat{\Bcal})\times_3\bfM_3^{-1}\times\dots\times_p\bfM_p^{-1}
    \]
     where
\[\hat{\Acal}=\Acal\times_3\bfM_3\times_4\bfM_4\times\dots\times_p\bfM_p \] .
\end{definition}

\begin{algorithm} 
\caption{$\starM$-product}
\label{alg:starMProd} 
 \begin{algorithmic}[1]
   \STATE \textbf{Inputs:} $\Acal\in\mathbb{R}^{n_1\times m\times n_3\times\dots\times n_p},\Bcal\in\mathbb{R}^{m\times n_2\times n_3\times\dots\times n_p}$, invertible $\bfM_3 \in \Rbb^{n_3\times n_3}, \dots, \bfM_p \in \Rbb^{n_p\times n_p}$
   \STATE $\hat{\Acal}=\Acal, \hat{\Bcal}=\Bcal$
   \FOR{$k=3,\dots,p$}
   \STATE $\hat{\Acal}=\hat{\Acal}\times_k \bfM_k$ , $\hat{\Bcal}=\hat{\Bcal}\times_k \bfM_k$
   \ENDFOR
   \STATE $\Ccal=(\hat{\Acal}\triangle\hat{\Bcal})$
   \FOR{$k=3,\dots,p$}
   \STATE $\Ccal=\Ccal\times_k \bfM_k^{-1}$
   \ENDFOR
   \STATE \textbf{Outputs:} $\Ccal\in\mathbb{R}^{n_1\times n_2\times\dots\times n_p}$
\end{algorithmic}
\end{algorithm}

Similar to how applying the Fourier Transform to a matrix of data will decouple relationships into frequencies, the purpose of the matrices  $\bfM_3,\dots,\bfM_p$ is to place a tensor in a transform domain where the higher-order relationships after the second dimension have been decoupled. The choice of $\bfM$ is left as a selectable parameter in the $\starM$-product. This introduces a means of flexibility into the framework in which one can strategically choose a transformation based on the nature of the data (e.g. the Discrete Fourier Transform for time series data). An additional important detail is that under the $\starM$-product, the lateral slices of a tensor are analogous to the columns of a matrix, and tubes are analogous to scalars.

\begin{figure}[h]
    \centering
    \includegraphics[width=1\linewidth]{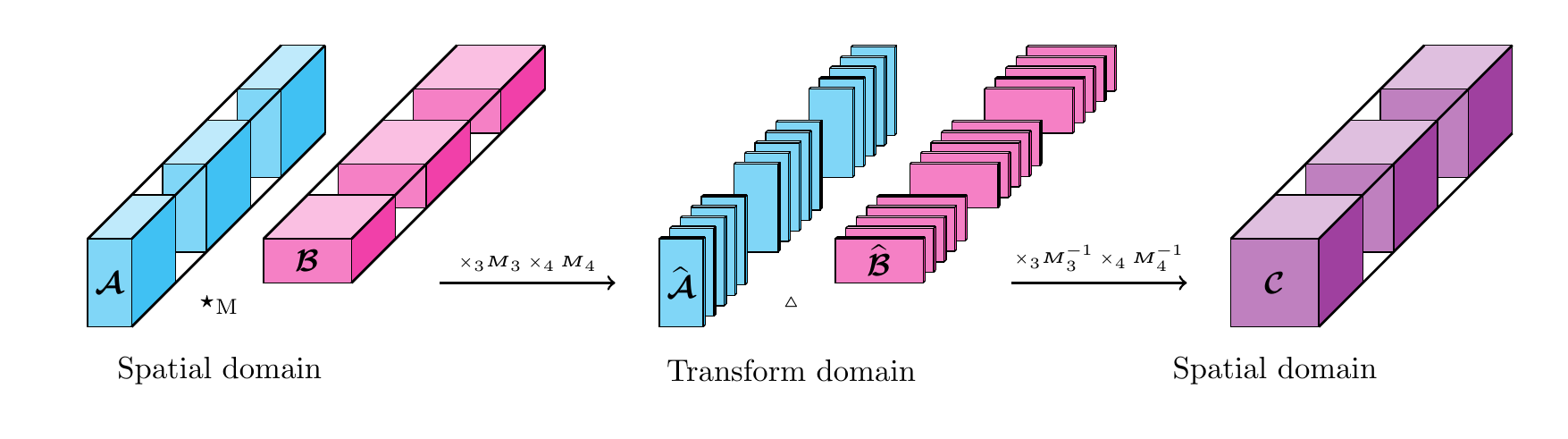}
    \caption{Visualization of $\starM$-product for 4D tensors}
    \label{fig:starM_visualization}
\end{figure}

\begin{example}\label{exam:starm} 
Given $\Acal \in \Rbb^{2\times 2 \times 2}$ and $\vec{\Bcal}\in \Rbb^{2\times 1\times 2}$ where
    \begin{align*}
        \Acal(:,:,1) &= \begin{bmatrix} 1 & 2 \\ 0 & -1\end{bmatrix} & \vec{\Bcal}(:,:,1) &= \begin{bmatrix} -1 \\ 1 \end{bmatrix}\\
        \Acal(:,:,2) &= \begin{bmatrix} -1 & 1 \\ 1 & 1\end{bmatrix} & \vec{\Bcal}(:,:,2) &= \begin{bmatrix} 0 \\ 1 \end{bmatrix}.
    \end{align*}

We choose the following $\bfM$:
\begin{align*}
\bfM&=\begin{bmatrix}
3 & 2 \\
1 & 1
\end{bmatrix} ,&
\bfM^{-1}&=\begin{bmatrix}
1 & -2 \\
-1 & 3
\end{bmatrix}
\end{align*}
We then move each tensor into the transform domain by taking the the mode-3 product of $\Acal$ and $\vec{\Bcal}$ with $\bfM$; that is, $\hat{\Acal} = \Acal \times_3 \bfM$ and $\hat{\vec{\Bcal}} = \vec{\Bcal}\times_3 \bfM$ which are
    \begin{align*}
        \hat{\Acal}(:,:,1) &= \begin{bmatrix} 1 & 8 \\ 2 & -1\end{bmatrix} & \hat{\vec{\Bcal}}(:,:,1) &= \begin{bmatrix} -3 \\ 5 \end{bmatrix}\\
        \hat{\Acal}(:,:,2) &= \begin{bmatrix} 0 & 3 \\ 1 & 0 \end{bmatrix} & \hat{\vec{\Bcal}}(:,:,2) &= \begin{bmatrix} -1 \\ 2 \end{bmatrix}
    \end{align*}

Now, we can take the facewise product of the two tensors to produce $\hat{\vec{\Ccal}} = \hat{\Acal} \triangle \hat{\vec{\Bcal}}$. The frontal slices obtained by computing this facewise product are described below: 

\begin{align*}
\hat{\vec{\Ccal}}(:,:,1) &= \begin{bmatrix} 37 \\ -11 \end{bmatrix} & \hat{\vec{\Ccal}}(:,:,2) &= \begin{bmatrix} 6 \\ -1 \end{bmatrix}
\end{align*}

Finally, we take the mode-3 product of $\bfM^{-1}$ with the tensor we computed above in order to obtain our final tensor $\Ccal = \hat{\Ccal} \times_3 \bfM^{-1}$, which is described below:

\begin{align*}
\Ccal(:,:,1) &= \begin{bmatrix} 25 \\ -9 \end{bmatrix} & \Ccal(:,:,2) &= \begin{bmatrix} -19 \\ 8 \end{bmatrix} 
\end{align*}

\end{example}

We also provide definitions of the $\starM$-transpose, the $\starM$-identity, and the notion of $\starM$-orthogonality. Note that these concepts arise directly from the $\starM$-product and are thus unique to this tensor framework.

\begin{definition}[$\starM$-transpose]\label{def:transpose}
$\Acal^\top\in\mathbb{R}^{n_2\times n_1\times\dots\times n_p}$ of $\Acal\in\mathbb{R}^{n_1\times n_2\times\dots\times n_p}$ is formed by transposing the frontal slices of $\Acal$ in the transform domain, or
\[
( \hat{\Acal}^{\top} ) (:,:,i_{3}, ..., i_{p} ) = \left( \hat{\Acal}(:,:,i_{3}, ..., i_{p} ) \right)^{\top} \text{ for } i_k = 1,\dots,n_k, \text{ where } k = 3,\dots,p.
\]
\end{definition}

\begin{definition}[$\starM$-identity]\label{def:identity}
The identity tensor $\Ical\in\mathbb{R}^{n\times n\times\dots\times n_p}$ is the tensor such that for any tensor $\Acal$, 
\[
\Acal\starM\Ical=\Ical\starM\Acal=\Acal.
\]
Note that in the transform domain, we have
\begin{align*}
    \widehat{\Ical}_{(:,:,i_3,\dots,i_p)}=\bfI \quad \text{ for } i_k = 1,\dots, n_k \text{ , } \text{ where } k=3,\dots, p.
\end{align*}
\end{definition}

\begin{definition}[$\starM$-orthogonality]\label{def:orthogonal}
A tensor $\Qcal\in\mathbb{R}^{n\times n\times\dots\times n_p}$ is orthogonal if
\begin{align*}
\Qcal^\top\starM \Qcal=\Qcal\starM \Qcal^\top=\Ical.
\end{align*}
\end{definition}

\subsection{t-SVDM}
\label{sec:tsvdm}

From the concepts defined in \Cref{sec:notation} and \Cref{sec:tensorTensorProducts}, we can introduce the t-SVDM, a higher-dimensional analog of the SVD based upon the $\starM$-product. The t-SVDM and its properties are described in \cite{kilmer2019tensortensor} for third-order tensors. In this section, we first provide the mathematical definition of the t-SVDM for tensors of order $p$ and describe its computation. We then highlight the important properties offered by the t-SVDM that make it ideal for the analysis of multidimensional data, similar to the benefits offered by the traditional SVD for matrix data.

First, we define the matrix SVD, which is described in further detail in \cite{strang}.

\begin{definition}[SVD]
\label{def:svd}
For a matrix $\bfA \in \mathbb{R}^{n_{1} \times n_{2} }$, the SVD is given by the decomposition
\begin{align*}
\bfA = \bfU \bfSigma \bfV^{\top},
\end{align*}
where $\bfU \in \mathbb{R}^{n_{1} \times n_{1}}$ and $\bfV \in \mathbb{R}^{n_{2} \times n_{2}}$ are orthogonal and $\bfSigma \in \mathbb{R}^{n_{1} \times n_{2}}$ is a diagonal matrix with all positive entries $\sigma_{i}$ for $i = 1, ... ,r$. These entries lie along the diagonal in descending order such that $\sigma_{1} \geq \sigma_{2} \geq \cdots \geq \sigma_{r} > 0$, where $r$ is the number of nonzero singular values.
\end{definition}

If we know that the matrix $\bfA$ is of rank $r$, then we can also express this decomposition as the sum of $r$ rank-1 matrices formed from the SVD matrices. Here, $\mathbf{u}_{i}$ and $\mathbf{v}_{i}$ represent the $i$-th column of $\bfU$ and $i$-th column of $\bfV$, respectively.
\begin{align*}
\bfA = \sum_{i=1}^{r} \sigma_i \mathbf{u}_{i} \mathbf{v}_{i}^\top
\end{align*}
Now that we have reviewed the SVD for matrices, we can introduce the t-SVDM.
\begin{definition}[t-SVDM]\label{def:tsvdm}
For a tensor $\Acal \in \mathbb{R}^{n_1 \times n_2 \times \dots \times n_p}$, the t-SVDM is given by the decomposition
\begin{align*}
\Acal = \Ucal \starM \Scal \starM \Vcal^\top,
\end{align*}
where $\Ucal \in \mathbb{R}^{n_1 \times n_1 \times n_3 \times \dots \times n_p }$ and $\Vcal \in \mathbb{R}^{n_2 \times n_2 \times n_3  \times \dots \times n_p }$ are $\starM$-orthogonal, and\linebreak $\Scal \in \mathbb{R}^{n_1 \times n_2 \times n_3  \times \dots \times n_p }$ is f-diagonal.
We denote $\mathbf{s}_{i}$ as the ($i$,$i$)-tube of $\Scal$. The tubes lie in descending magnitude such that $\|\mathbf{s}_{1}\|_F \geq \|\mathbf{s}_{2}\|_F \geq \cdots \geq \|\bfs_{r}\|_F > 0$ where $r$ is the number of nonzero singular tubes.
\end{definition}

A visualization of the t-SVDM for a third-order tensor is provided in \Cref{fig:tsvdm}.

\begin{figure}[H]
    \centering
    \includegraphics[width=0.5\textwidth]{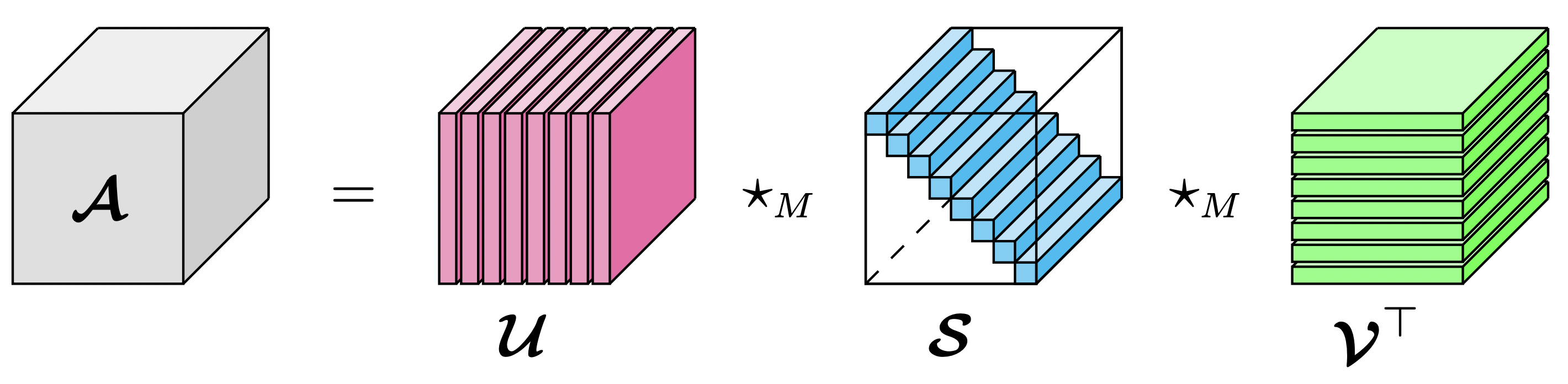}
    \caption{t-SVDM for third-order tensors}
    \label{fig:tsvdm}
\end{figure}

The t-SVDM also gives us a notion of rank for tensors, which we define in \Cref{def:t-rank}.

\begin{definition}[t-rank]\label{def:t-rank}
Let $\Acal \in \mathbb{R}^{n_1 \times n_2 \times \dots \times n_p}$, with its t-SVDM given by $\Ucal \starM \Scal \starM \Vcal^{\top}$. We say that the t-rank of $\Acal$ is equal to the number of nonzero tubes in $\Scal$, $r$.
\end{definition}

Suppose that $\Acal$ has t-rank-$r$. Then, similar to the matrix SVD case, we can rewrite $\Acal$ using the t-SVDM tensors as the sum of $r$ t-rank-1 tensors formed from the lateral slices of $\Ucal$ and $\Vcal$ and the tubes of $\Scal$:
$$
\Acal = \sum_{i=1}^{r} \vec{\Ucal}_{i} \starM \mathbf{s}_{ii} \starM \vec{\Vcal}_{i}^\top.
$$
We visualize this expansion and its relation to the t-SVDM tensors in \Cref{fig:rank1tensors}.

\begin{figure}[h]
    \centering
    \includegraphics[width=0.5\textwidth]{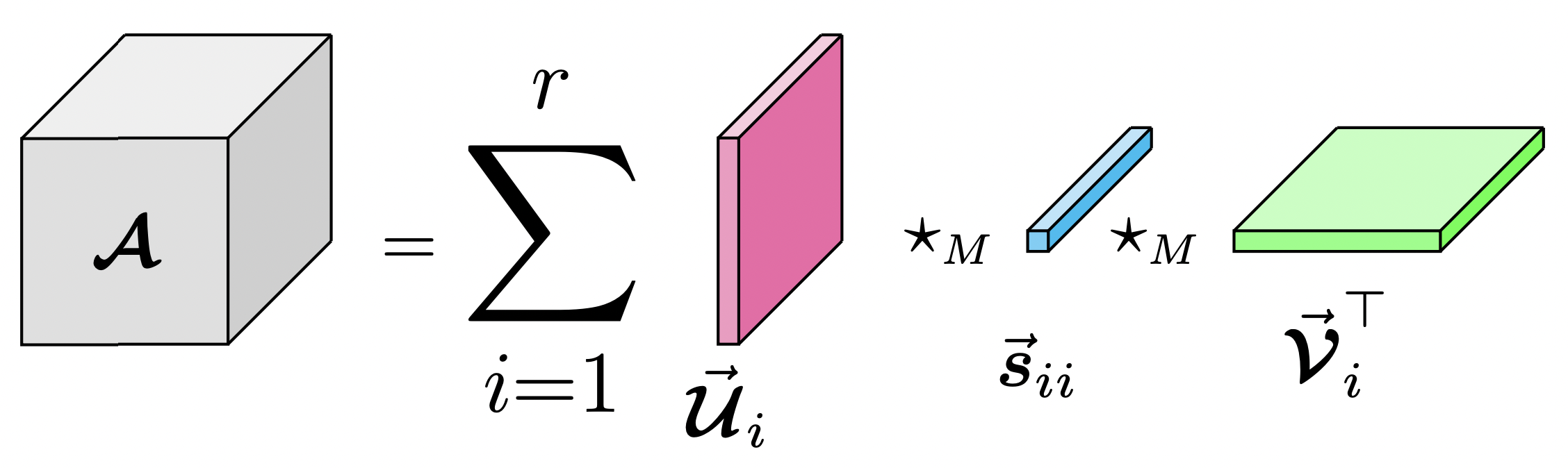}
    \caption{Expansion of t-rank-1 tensors derived from slices of $\Ucal$ and $\Vcal$ with diagonal tubes of $\Scal$}
    \label{fig:rank1tensors}
\end{figure}

We then introduce the implementation of t-SVDM in \Cref{alg:tsvdm}. Apart from the flexible choices of transformation matrices $\bfM$ as mentioned when discussing the $\starM$-product, the algorithm also has the advantage of allowing parallel computation as matrix SVD computations are not reliant on each other. This allows us to break the t-SVDM computation into smaller pieces. 
\begin{algorithm}[h]
\caption{t-SVDM}
\label{alg:tsvdm} 
 \begin{algorithmic}[1]
   \STATE \textbf{Input:} $\Acal \in \mathbb{R}^{n_1 \times n_2 \times \dots \times n_p}$, invertible $\bfM_3 \in \mathbb{R}^{n_3 \times n_3 }, ..., \bfM_p \in \mathbb{R}^{n_p \times n_p }$
   \STATE Move into transform domain: $\hat{\Acal} \leftarrow \Acal$
   \STATE Concatenate frontal slices along third dimension: $\hat{\Acal} = \reshape(\hat{\Acal},[n_1, n_2, n_3n_4\cdots n_p])$ 
   \FOR{$i = 1,\dots,(n_3n_4\cdots n_p)$}
        \STATE Compute matrix SVD:
            $\hat{\Acal}(:,:,i) = \hat{\Ucal}(:,:,i) \cdot \hat{\Scal}(:,:,i) \cdot \hat{\Vcal}(:,:,i)^\top$
   \ENDFOR
   \STATE Reshape into $p$-dimensional tensors:\\ $\hat{\Ucal} = \reshape(\hat{\Ucal}, [n_{1}, n_{1}, n_{3} ... , n_{p} ])$\\
   $\hat{\Scal} = \reshape(\hat{\Scal}, [n_{1}, n_{2}, n_{3} ... , n_{p}])$\\
   $\hat{\Vcal} = \reshape(\hat{\Vcal}, [n_{2}, n_{2}, n_{3} ... , n_{p} ])$
   \STATE Move back to original domain: $\Ucal \leftarrow \hat{\Ucal},  \Vcal \leftarrow \hat{\Vcal} ,  \Scal \leftarrow \hat{\Scal}  $\\
   \STATE \textbf{Output:} $\Ucal \in \mathbb{R}^{n_1 \times n_1 \times \dots \times n_p }$, $\Vcal \in \mathbb{R}^{n_2 \times n_2 \times \dots \times n_p }$, $\Scal \in \mathbb{R}^{n_1 \times n_2 \times \dots \times n_p }$ 
\end{algorithmic}
\end{algorithm}

The many useful properties of the matrix SVD are well-understood and documented \cite{strang}. We now highlight some similar beneficial properties offered by the t-SVDM.  

\textbf{Basis:} For matrices, one can take a linear combination of basis vectors to produce a new vector. In the matrix SVD, the columns of $\bfU$ form a basis for the range or column space of $\bfA$.  This means that for each column $j$, we can find scalars $c_1, ..., c_r \in \mathbb{R}$ such that
\begin{align*}
    \bfA(:,j) = c_1\bfU(:,1) + \dots + c_r \bfU(:,r).
\end{align*}
As mentioned earlier, the tensor analogs for columns and vectors are the lateral slices. The lateral slices of $\Ucal$ (the $\vec{\Ucal}_{i}$ slices) form a basis for approximating the original tensor $\Acal$. The lateral slices of $\Acal$ can be obtained by taking a tensor linear combination (or t-linear combination), in which one computes the $\starM$-product of the lateral slices with tubes (the analog to scalars in the $\starM$-framework) \cite{Kilmer2013:tensorOperators}. 
For each lateral slice $j$, we can find tubes $\mathbf{c}_1, ..., \mathbf{c}_r \in \mathbb{R}^{1 \times 1 \times n_{3} \times \dots \times n_{p}}$ such that
\begin{align*}
    \vec{\Acal}_j =  \vec{\Ucal}_{1} \starM \mathbf{c}_1 + \dots +  \vec{\Ucal}_{r} \starM \mathbf{c}_r.
\end{align*}

\textbf{Eckart-Young Theorem in $\starM$-framework:} The Eckart-Young theorem \cite{strang} states that for any matrix $\bfB$ with rank $k$, we have $\| \bfA - \bfA_{k} \|_{F} \leq \| \bfA - \bfB \|_{F}$. An extension of  the Eckart-Young theorem also exists for t-rank-$k$ approximations obtained using the t-SVDM. 
Note that the following theorem simply extends work that have been provided in \cite{kilmer2019tensortensor} to tensors of order-$p$ rather than third order.

\begin{theorem}[Eckart-Young for higher-order $\starM$-product]\label{thm:eckartYoung}
Let $\Acal \in \Rbb^{n_1\times n_2\times \dots \times n_p}$ be a order-$p$ tensor with the full t-SVDM $\Acal = \Ucal \starM \Scal \starM \Vcal^\top$ where the $\starM$-product consists of only multiples of orthogonal transformations.
Define $\Acal_k = \Ucal_k \starM \Scal_k \starM \Vcal_k^\top$ as the t-rank-$k$ approximation of $\Acal$ obtained through truncation. 
Then, 
    \begin{align*}
        \Acal_k &= \argmin_{\Bcal\in \Rbb^{n_1\times n_2\times \dots \times n_p}} \|\Acal - \Bcal\|_F\quad \st \quad\trank(\Bcal)=k.
    \end{align*}
Furthermore, the squared error is given by
\begin{align*}
        \|\Acal - \Acal_k\|_F^2 = \sum_{i=k+1}^r \|\bfs_{i}\|_F^2
\end{align*}
where $\bfs_i$ is the $(i,i)$-tube of $\Scal$ and $r$ is the t-rank of $\Acal$.
\end{theorem}

\begin{proof}
See \Cref{sec:proofs_appendix}.
\end{proof}

\textbf{Optimal t-rank-$k$ Tensor Representation over rank-$k$ Matrix Representation:} Another important property is the provable optimality of a t-rank-$k$ approximation of a tensor $\Acal$ obtained using the t-SVDM compared to a rank-$k$ approximation obtained from a matrix $\bfA$ containing the vectorized information of $\Acal$, or
\begin{align*}
\| \Acal - \Acal_{k} \|_{F} \leq \| \bfA - \bfA_{k} \|_{F} .
\end{align*}
Here, $\Acal_k$ and $\bfA_{k}$ represent a t-rank-$k$ tensor and a rank-$k$ matrix, respectively. The proof for this is provided in \cite{kilmer2019tensortensor} and relies upon the Eckart-Young theorem for tensors as defined earlier. For our work, this fact illustrates that representing inherently high-dimensional data as a tensor is provably optimal to the corresponding matrix representation. In the following section, we introduce a classification approach based upon the $\Ucal$ obtained using the t-SVDM.

\section{Local Tensor SVD Approaches for Classification}
\label{sec:local}

Classification tasks rely on two basic assumptions: data from the same class share common features and data from distinct classes have different fundamental features. 
The crux of classification algorithms is the method by which one extracts meaningful features from the data.  
In this work, we consider supervised classification tasks, where the class labels are known a priori. In our case, the class labels are whether the subject is viewing a picture or reading a sentence. We assume that fMRI trials with either one of these class labels share some fundamental commonalities, i.e. that brains viewing a picture are responding in a distinct way from which they would respond while reading a sentence. We also assume that such differences are detectable using fMRI.

\subsection{Algorithm Overview}
\label{sec:localTSVD}

\begin{figure}[h]
    \centering
    
    \includegraphics{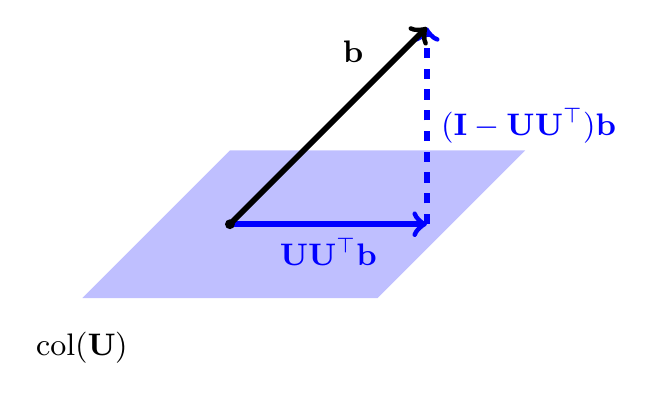}
    \includegraphics{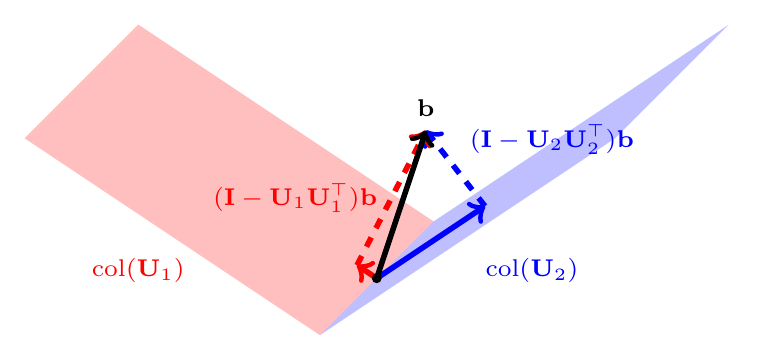}

    \caption{Illustration of a single orthogonal projection in $\Rbb^3$ (left) and using orthogonal projections in $\Rbb^3$ to classify (right). 
    In the left illustration, assume $\bfU$ is a matrix with orthonormal columns. 
    The projection (blue solid), $\bfU\bfU^\top \bfb$, lies in the column space of $\bfU$ and the error (blue dashed), $(\bfI - \bfU\bfU^\top)\bfb$, is orthogonal to the projection.
    In the right illustration, the vector $\bfb$ is orthogonally projected onto the column spaces of $\bfU_1$ and $\bfU_2$ respectively.
    The vector $\bfb$ lies more in $\col(\bfU_2)$ than in $\col(\bfU_1)$. 
    Hence, $\bfb$ would be classified as belonging to the subspace spanned by $\bfU_2$.}
    \label{fig:projectionIllustration}
\end{figure}

Our work uses a projection-based classification approach as presented in \cite{newman2017image}. We generalize this approach to higher-order tensors and a family of tensor-tensor products. 
This approach begins by extracting features from our training data through building a local basis for each class.  
Then, we orthogonally project a test image onto the spaces spanned by each local basis, as illustrated in \Cref{fig:projectionIllustration} in $\Rbb^3$.  
We make a classification decision based on the class of the local basis for which the projection produces the smallest norm difference with the original test image.
The key to successful projection-based classification is finding a representative basis for each class that simultaneously captures the fundamental features of the class and distinguishes between distinct classes.  

Due to its provably optimal representation and the flexible choice of transformation, we propose forming local bases and projections using the t-SVDM. 
Specifically, let $\Acal_i$ contain the fMRI data for the $i$-th class as lateral slices of $\Acal_i$; that is, $\Acal_i$ is a fifth-order tensor of size $(x, \textup{trials}_i ,  y , z , \textup{time})$ where $x$, $y$, and $z$ correspond to the spatial dimensions\footnote{Note we can permute the dimensions of the class tensor freely as long as the second dimension corresponds to the number of images.}. We compute its t-SVDM
    \begin{align*}\label{eq:localTSVDM}
        \Acal_i = \Ucal_i \starM \Scal_i \starM \Vcal_i^\top.
    \end{align*}

We choose the second dimension to be the trials so that each lateral slice of $\Acal_i$ is one fMRI image.  
Hence, the lateral slices of $\Ucal_i$ form an orthonormal basis which contain the most important features of the fMRI data across the trials.

To obtain features that are representative of class $i$, but distinguishable from other classes, we truncate the t-SVDM to $k$ terms\footnote{Note that we can also select a different truncation parameter $k$ for each class $i$.}
    \begin{align*}\label{eq:localTSVDMTruncated}
        \Acal_i \approx \Ucal_{i,k} \starM \Scal_{i,k} \starM \Vcal_{i,k}^\top
    \end{align*}
where $\Ucal_{i,k}$ contains the first $k$ lateral slices and $\Scal_{i,k}$ and $\Vcal_{i,k}$ are truncated accordingly. 
The truncated t-SVDM will be the best t-rank-$k$ approximation to $\Acal_i$ (see \Cref{thm:eckartYoung}) and hence the local basis $\Ucal_{i,k}$ is the best set of $k$ lateral slices to describe class $i$.  

\sloppy Suppose we have $c$ classes.  We form class tensors $\Acal_0,\dots, \Acal_{c-1}$ and local bases $\Ucal_{0,k_1}, \dots, \Ucal_{c-1,k_{c-1}}$ where each class can have its own truncation and choice of transformation.   
We orthogonally project a test fMRI image (i.e., an image not contained in the class tensors) stored as a lateral slice $(x, 1 , y , z , \textup{time})$ onto each of the class spaces via
    \begin{align*}\label{eq:projectionTSVDM}
        \vec{\Pcal}_i = \Ucal_{i,k_i}\starM \Ucal_{i,k_i}^\top \starM \vec{\Tcal}.
    \end{align*}
The projection $\vec{\Pcal}_i$ is a t-linear combination of the lateral slices contained in $\Ucal_{i,k_i}$, and hence lies in the span of $\Ucal_{i,k_i}$; see~\cite{Kilmer2013:tensorOperators, KernfeldKilmerAeron2015} for details. 
Analogous to the matrix case, $\vec{\Pcal}_i$ is the closest image lying in the span of $\Ucal_{i,k_i}$ to the original image $\vec{\Tcal}$. Here, closeness is measured in the Frobenius norm, although other norms can also be utilized.

After projecting the test image onto each of the spaces spanned by the local bases, we classify based on the projection that was closest to the original image; that is,
\begin{align*}
i^* = \argmin_{i=0,\dots,c-1} \|\vec{\Tcal} - \vec{\Pcal}_i\|_F.
\end{align*}
Here, $i^*$ is the predicted class.

The t-SVDM projection-based classification algorithm offers several advantages.
First, the projection-based method is simple and efficient to implement.  The local bases are pre-computed and can be formed in parallel because the class tensors are distinct. 
Second, the proposed algorithm is a direct method. We do not form a parameterized decision boundary and hence there is no training process to adjust parameters.  
Third, our method is flexible.   We extend the original work in~\cite{newman2017image} to a more general family of tensor-tensor products based on the choice of $\bfM$. 
In doing so, we offer many choices of transformations that, when well-chosen, can improve the classification results. 
Fourth, this algorithm is based upon a rigorous tensor algebraic framework created by the $\starM$-product and is therefore mathematically justified. As seen in \Cref{thm:eckartYoung}, the t-SVDM satisfies an Eckart-Young-like Theorem, hence the local bases we form are in some sense optimal. This also gives us a natural analog to projections in multidimensional space.

\subsection{Intuition Example for Algorithm}
\label{sec:mnist}
This following example is intended to serve as a stepping stone towards understanding this t-SVDM classification algorithm before we proceed to describing our application to fMRI data. The MNIST database consists of 70000 $28 \times 28$ grayscale images where each contains one digit between 0 and 9, resulting in 10 possible classes \cite{lecun-mnisthandwrittendigit-2010}. For illustrative purposes, we apply the local t-SVDM algorithm using only the first two classes (digits 0 and 1). 

\begin{center}
    \begin{figure}[h]
    \centering
    \begin{subfigure}[b]{0.25\textwidth}
    \centering
    \stackunder[5pt]{\includegraphics[trim={0cm 0 0 1cm},clip,width=0.8\textwidth]{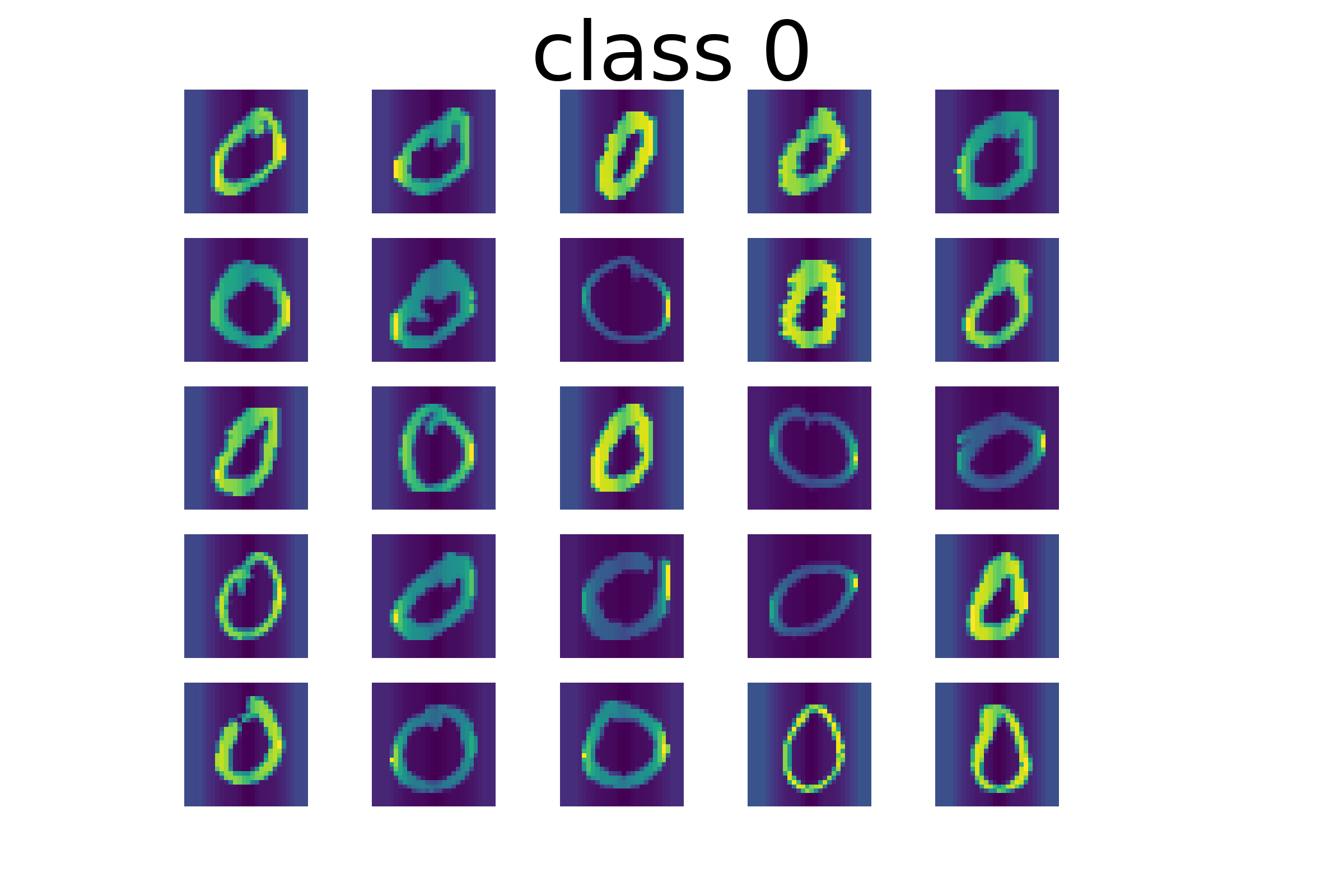}}{$\Acal_{0}$}
    \end{subfigure}
    \hspace{3cm}
    \begin{subfigure}[b]{0.25\textwidth}
    \centering
    \stackunder[5pt]{\includegraphics[trim={0cm 0 0 1cm},clip,width=0.8\textwidth]{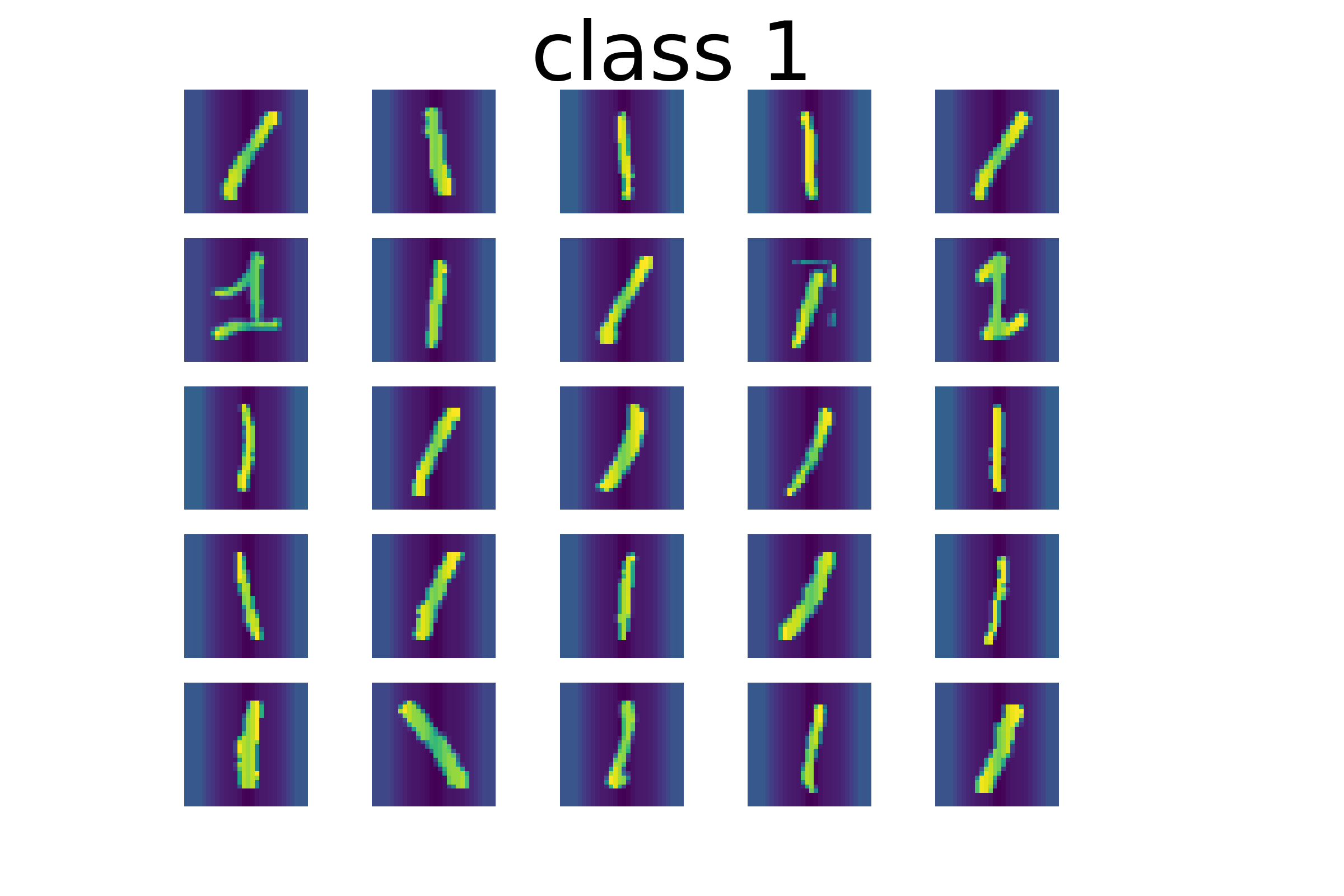}}{$\Acal_{1}$}
    \end{subfigure}\\
    \vspace{0.3cm}
    \centering
    \begin{subfigure}[b]{0.4\textwidth}
    \centering
    \stackunder[5pt]{\includegraphics[width=0.4\textwidth]{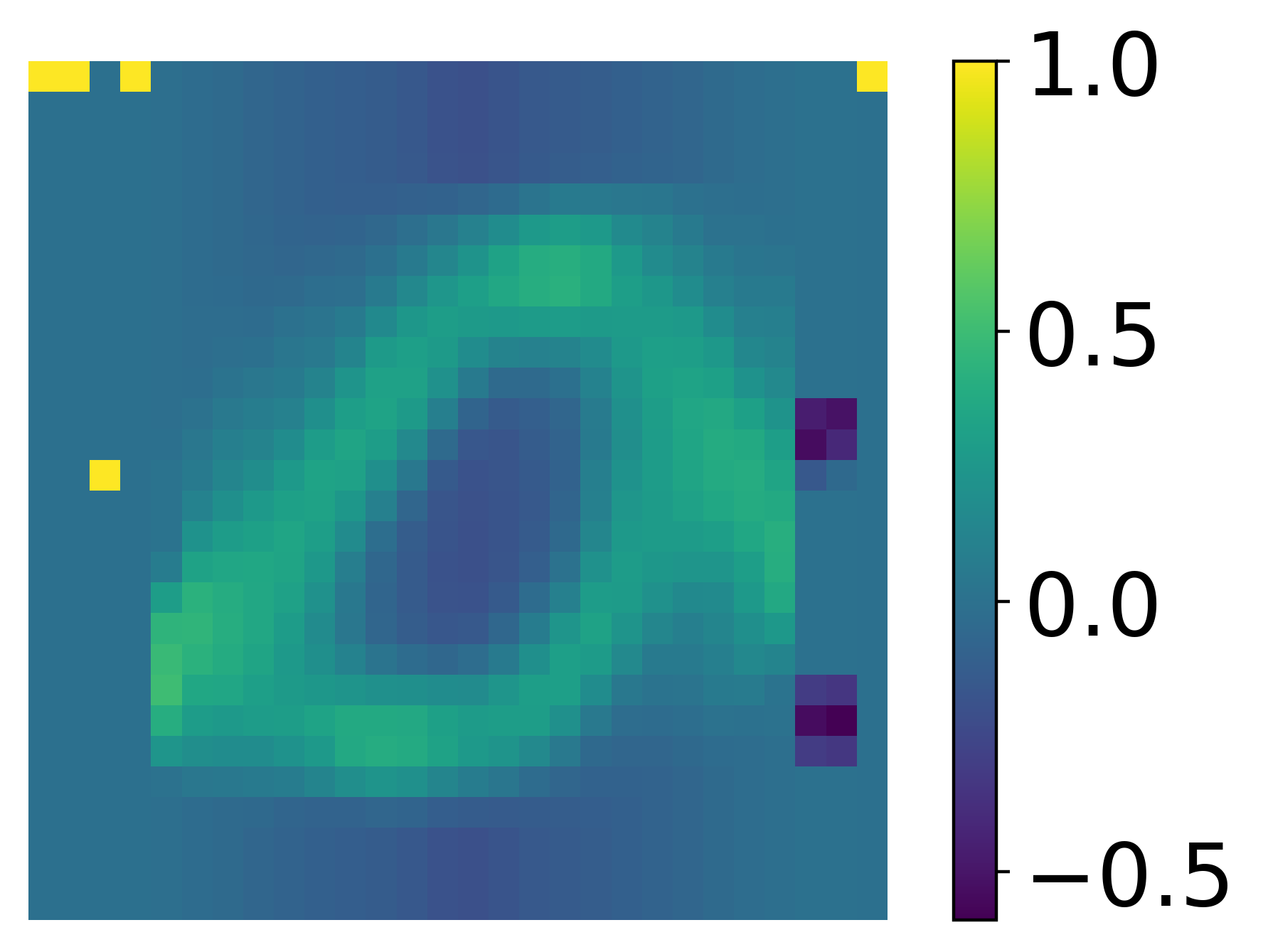}
    \includegraphics[width=0.4\textwidth]{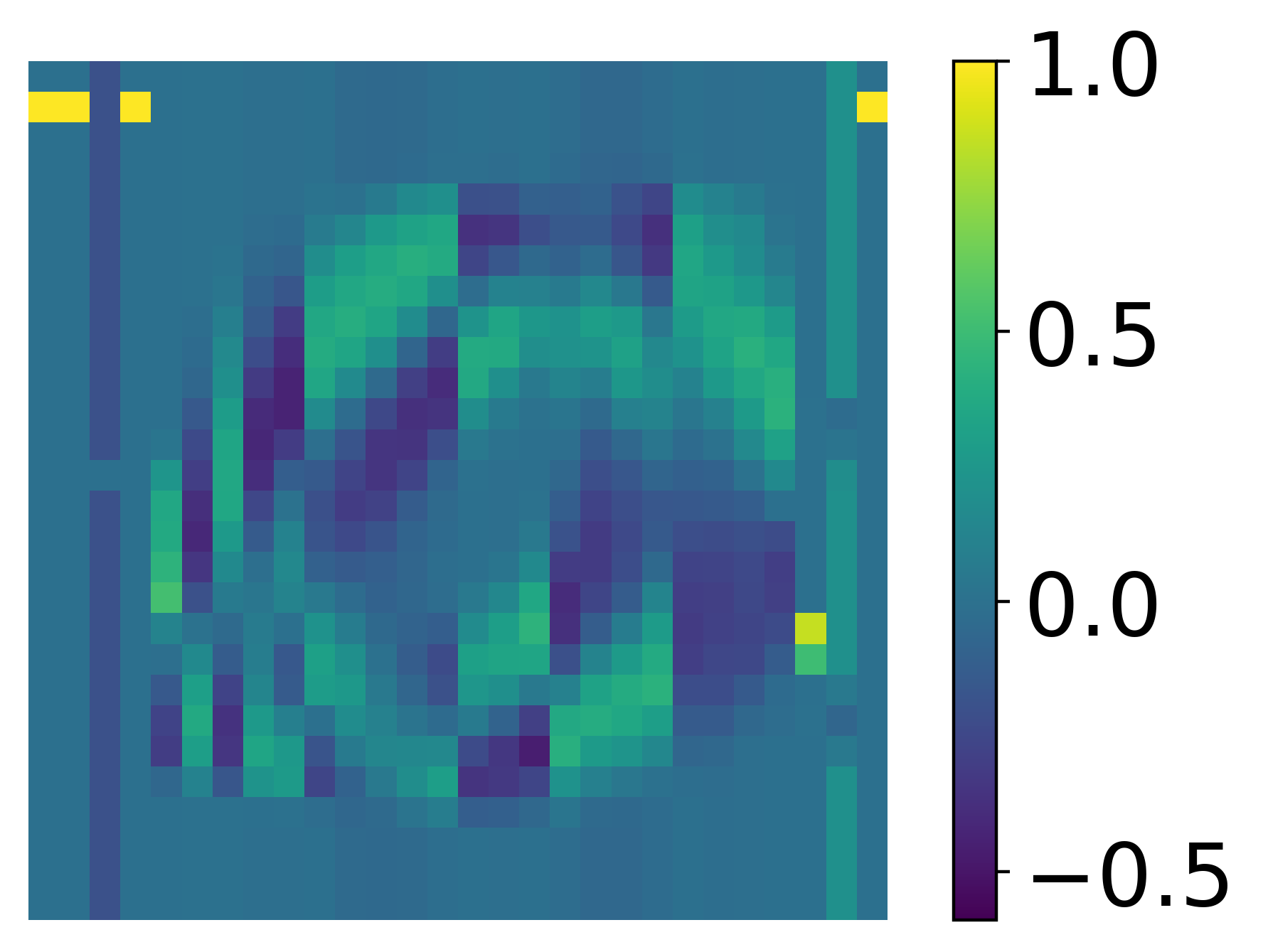}}{$\Ucal_{0,2}$}
    \end{subfigure}
    \hspace*{2em}%
     \begin{subfigure}[b]{0.4\textwidth}
    \centering
    \stackunder[5pt]{\includegraphics[width=0.4\textwidth]{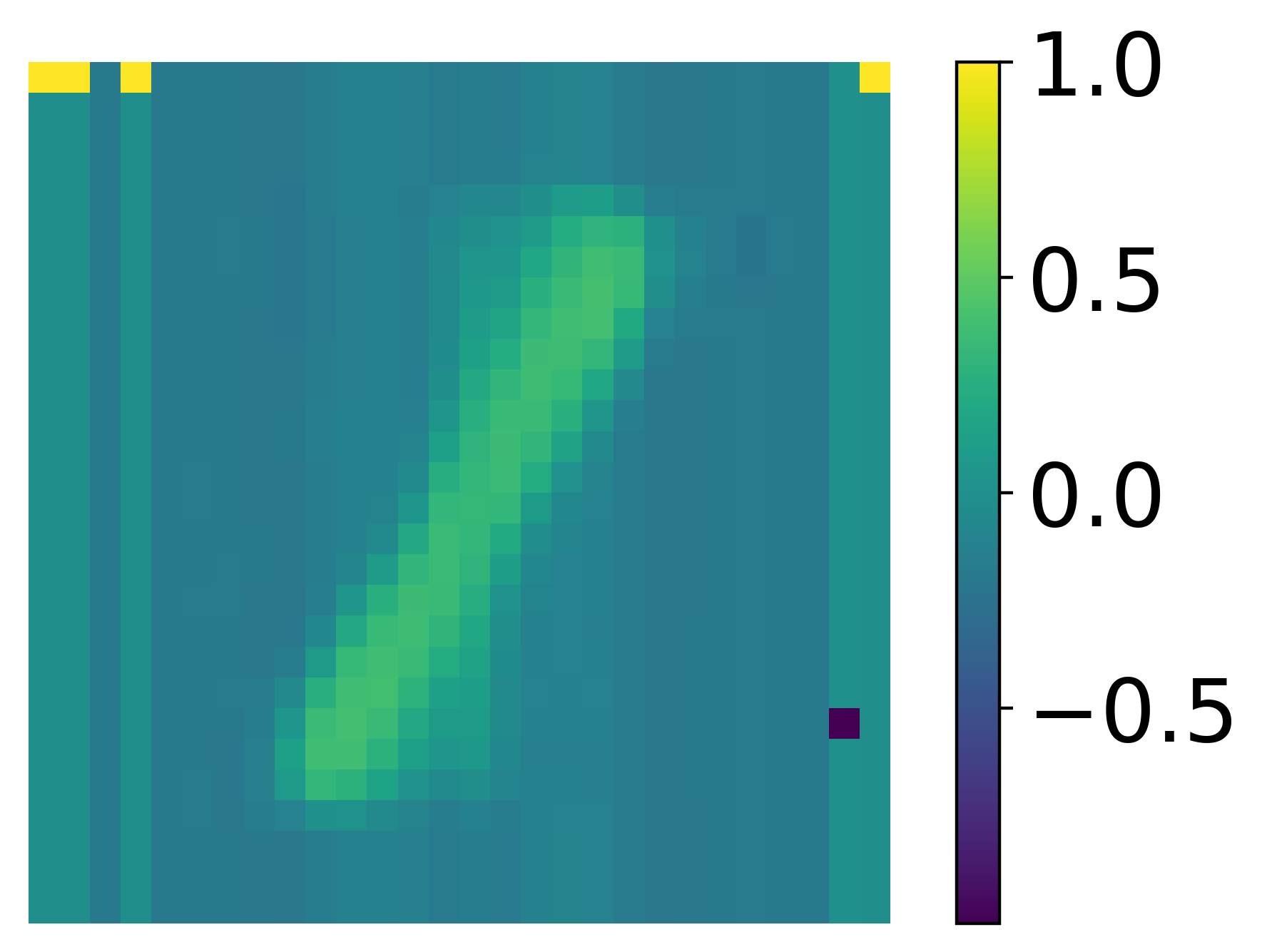}
    \includegraphics[width=0.4\textwidth]{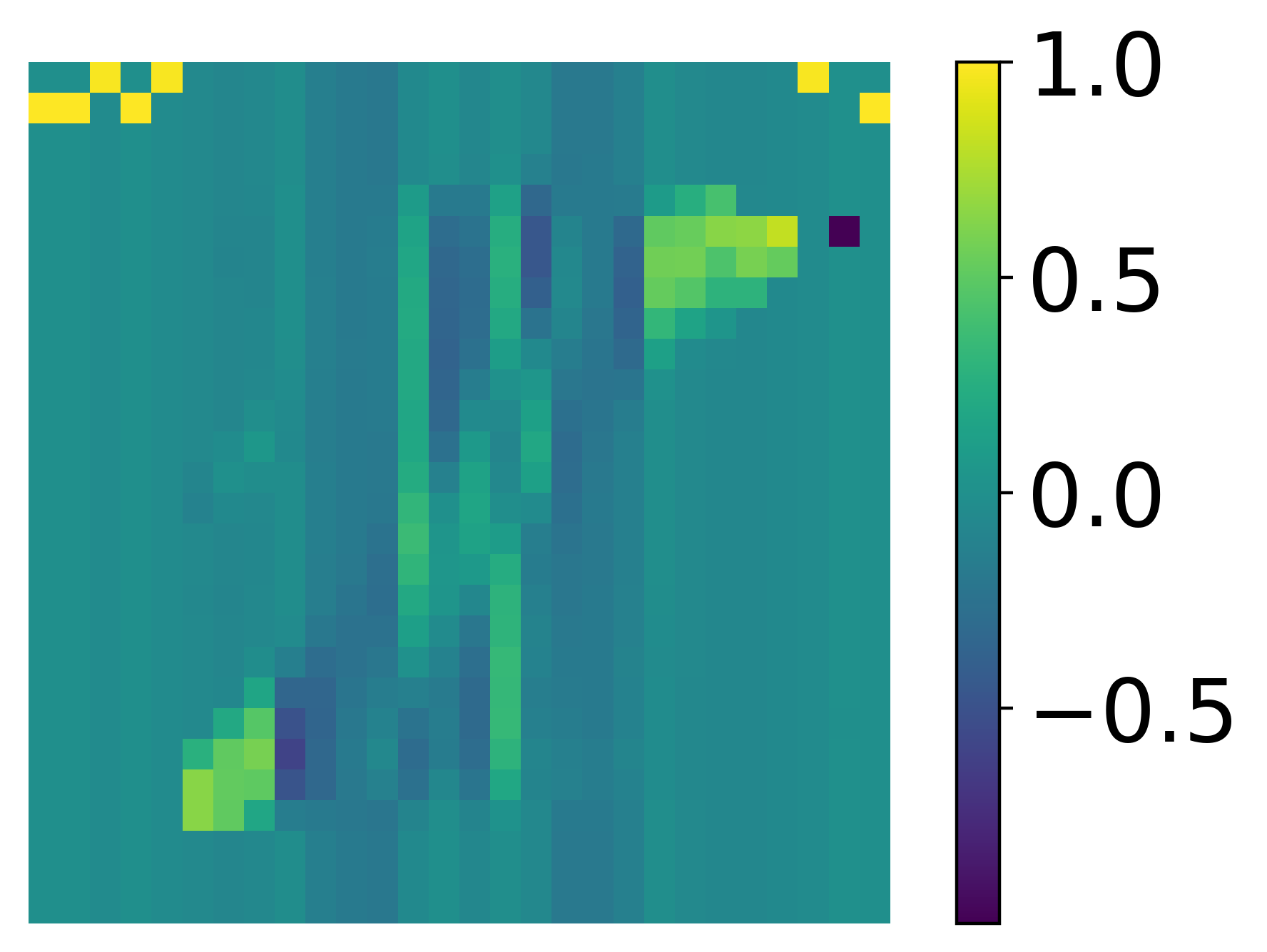}}{$\Ucal_{1,2}$}
    \end{subfigure}\\
    \vspace{0.3cm}
    \centering
    \begin{subfigure}[b]{0.18\textwidth}
    \centering
    \stackunder[5pt]{\includegraphics[width=0.9\textwidth]{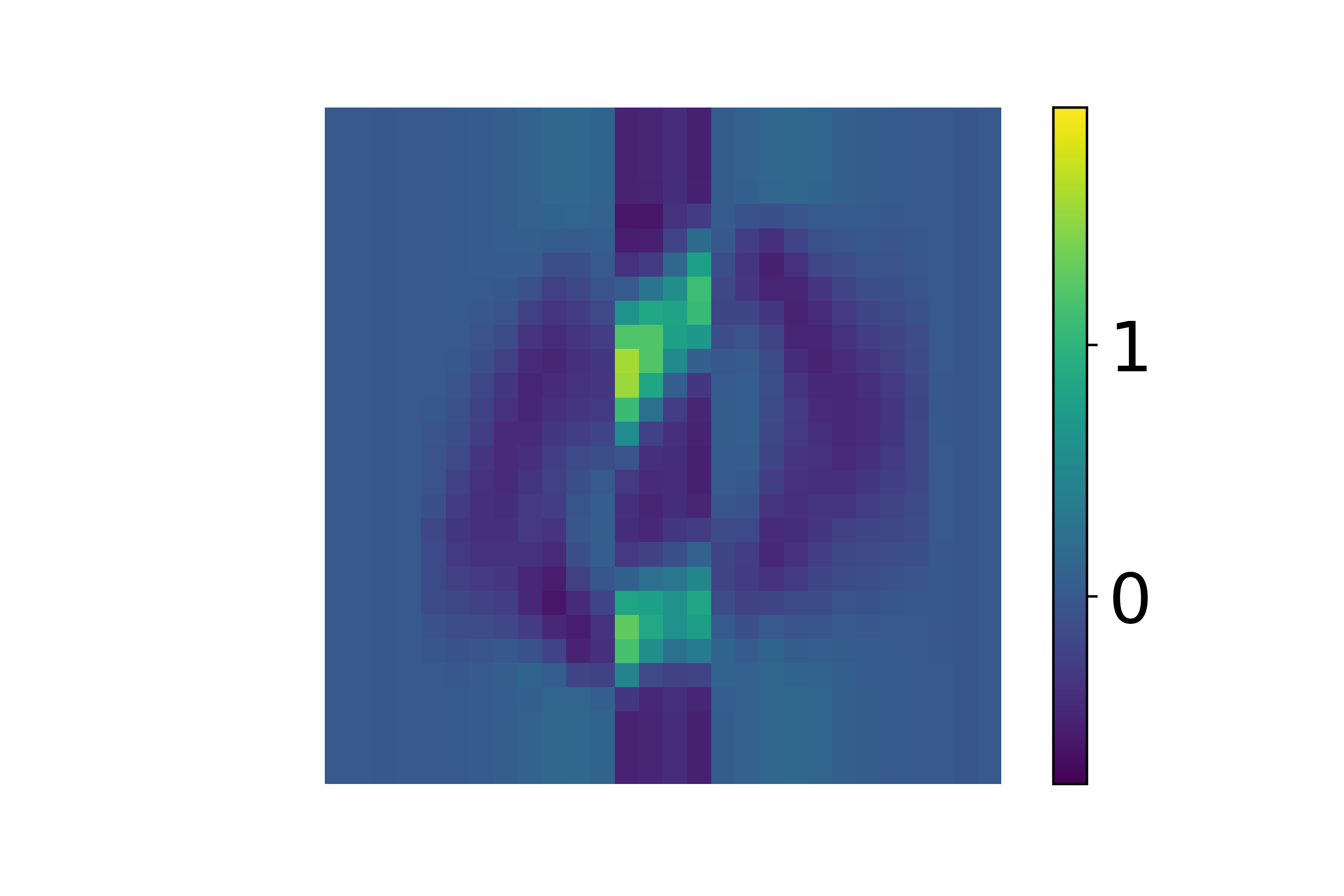}}{$\Pcal_{0}$}
    \end{subfigure}
    \begin{subfigure}[b]{0.18\textwidth}
    \centering
    \stackunder[5pt]{\includegraphics[width=0.9\textwidth]{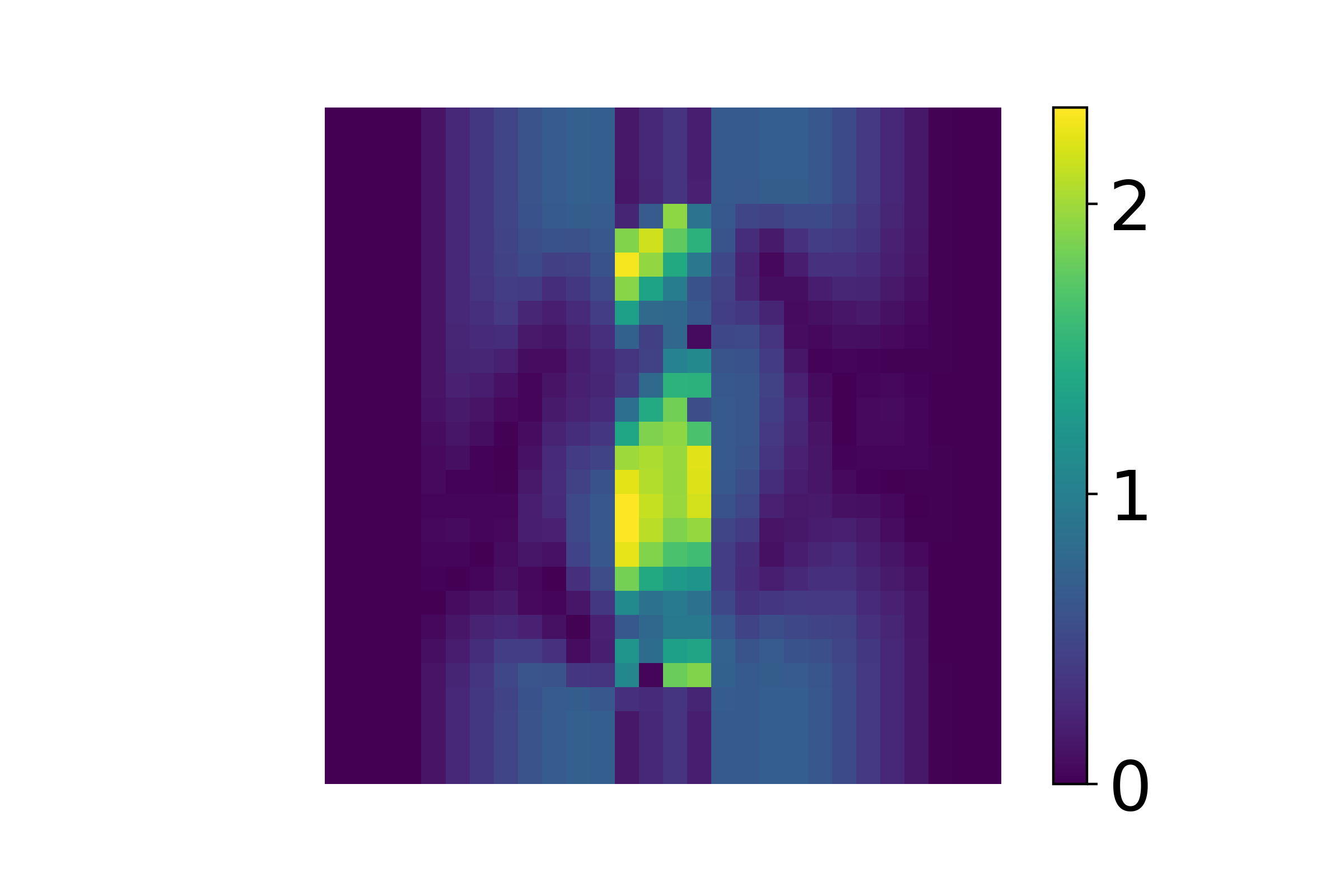}}{$|\Tcal - \Pcal_{0}|$}
    \end{subfigure}
    \begin{subfigure}[b]{0.18\textwidth}
    \centering
    \stackunder[5pt]{\includegraphics[width=0.9\textwidth]{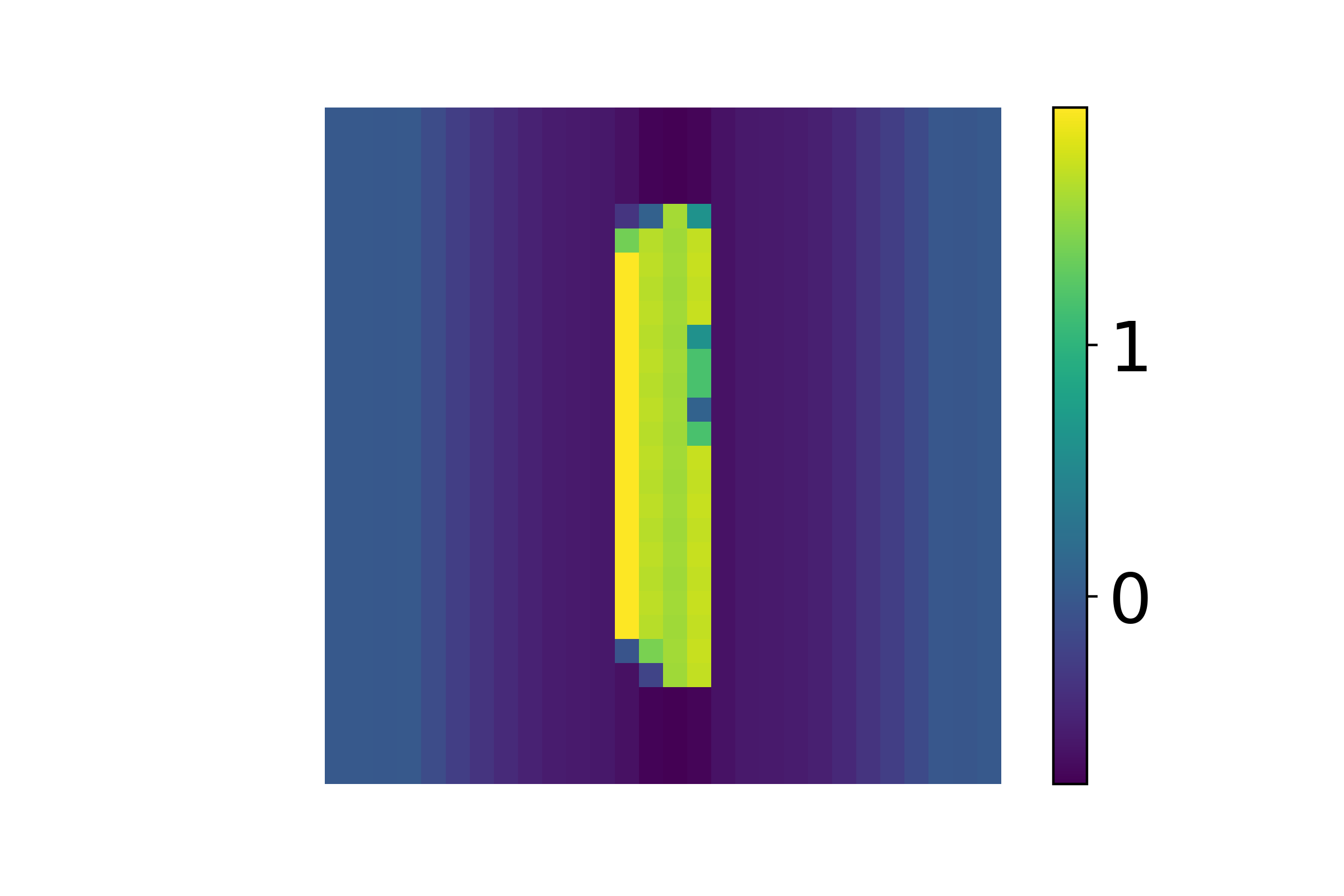}}{$\Tcal$}
    \end{subfigure}
    \begin{subfigure}[b]{0.18\textwidth}
    \centering
    \stackunder[5pt]{\includegraphics[width=0.9\textwidth]{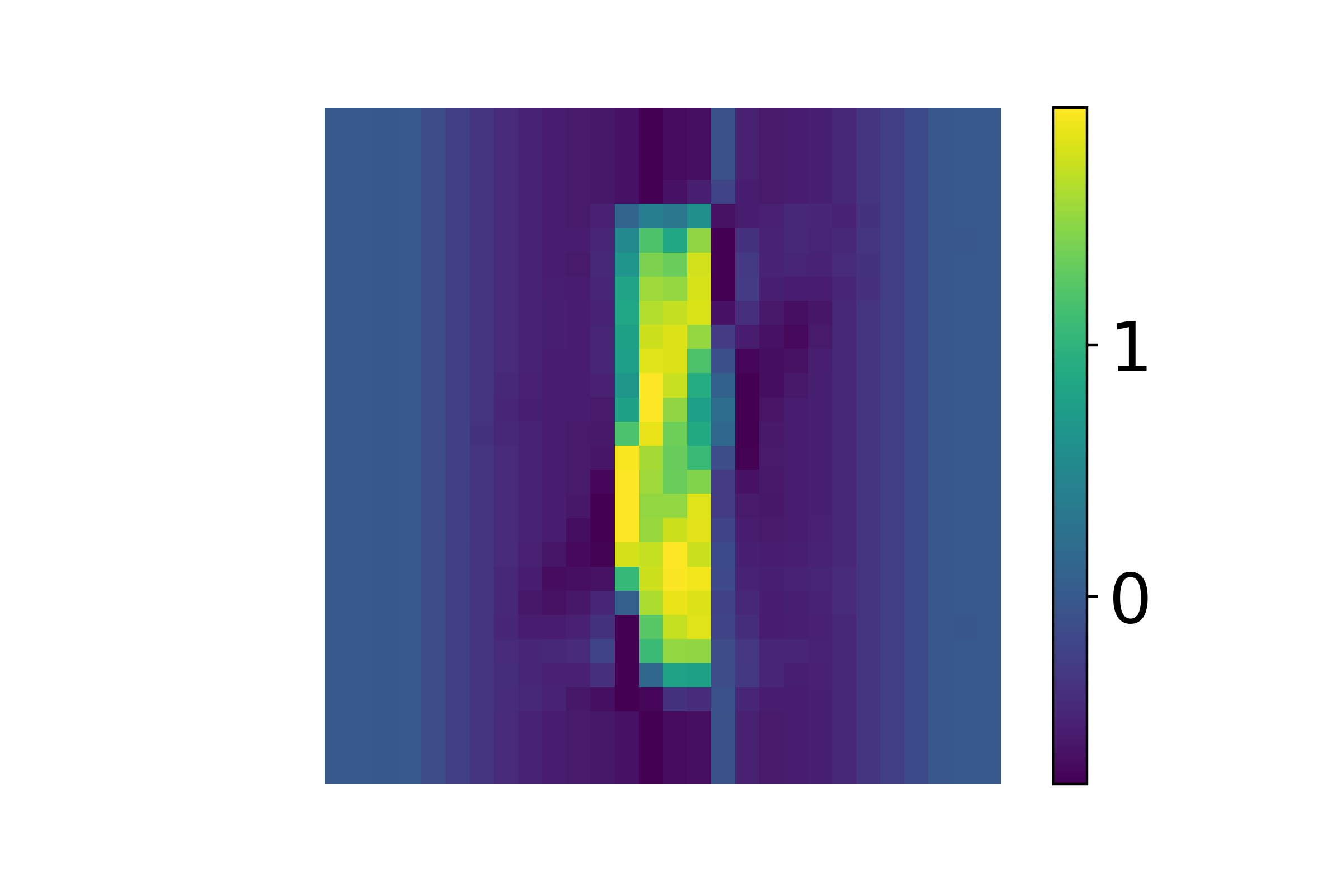}}{$\Pcal_{1}$}
    \end{subfigure}
    \hspace*{0.2em}
    \begin{subfigure}[b]{0.18\textwidth}
    \centering
    \stackunder[5pt]{\includegraphics[width=0.9\textwidth]{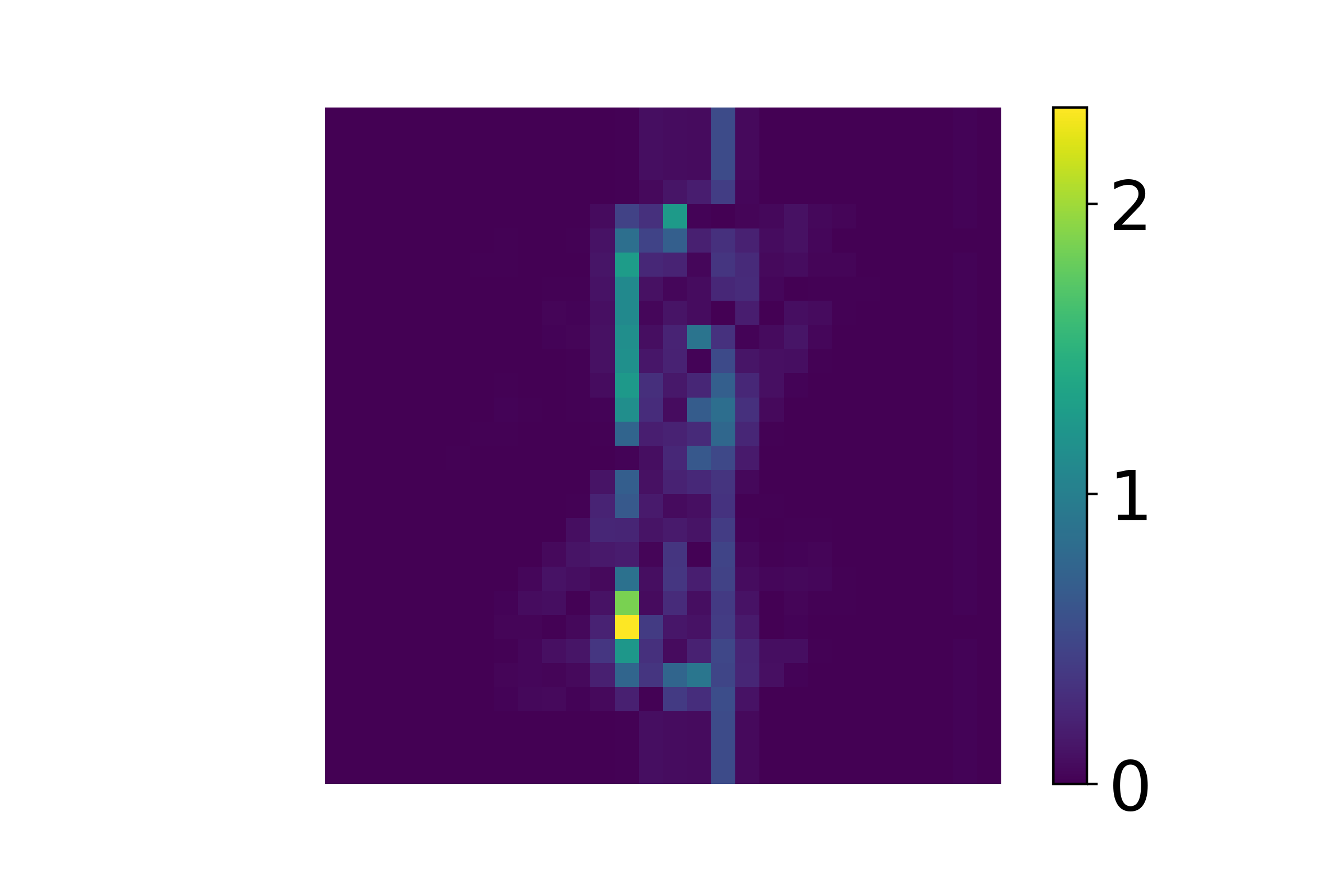}}{$|\Tcal - \Pcal_{1}|$}
    \end{subfigure}
        \caption{Illustration of applying local t-SVDM classification algorithm on MNIST database. We compute the t-SVDM of two class tensors $\Acal_0$ (representing digits consisting of 0) and $\Acal_1$ (representing digits consisting of 1), both with dimension $(\textup{x}, \textup{trials}, \textup{y}) = (28, 100, 28)$. Bases $\Ucal_{0,2}$ and $\Ucal_{1,2}$ are generated by class 0 and class 1 respectively.  $\Tcal$ represents the test image, which belongs to class $\Acal_1$. We project $\Tcal$ onto the spaces spanned by $\Ucal_{0,2}$ and $\Ucal_{1,2}$ and obtain projections $\Pcal_0$ and $\Pcal_1$ respectively. Absolute difference images $|\Tcal - \Pcal_0|$ and $|\Tcal - \Pcal_1|$ are generated by the absolute pixel difference between $\Tcal$ and $\Pcal_0$, and $\Tcal$ and $\Pcal_1$.}
        \label{fig:intuition}
    \end{figure}
\end{center}

\Cref{fig:intuition} provides an illustration of how classification via local tensor SVD (\Cref{sec:localTSVD}) is accomplished using MNIST data. Here, we select a small truncation value $k = 2$ to build a basis and utilize the t-product (or Discrete Foureir Transform) as our transformation. $\Ucal_{0,2}$ and $\Ucal_{1,2}$ both look very similar to the original digits, which capture the features of digits. $\Ucal_{0,2}$ exhibits the roundness of digit 0, and $\Ucal_{1,2}$ shows the vertical characteristics of digit 1. Shift of digits is also caught in the truncated basis because of our specific transformation choice.
Since $\Pcal_0$ is the projection of the test image $\Tcal$ onto the space spanned by the lateral slices of $\Ucal_{0,2}$, we visually observe that the projection is blurred, and it seemed to have characteristics of both digit 0 and digit 1. Conversely, $\Pcal_1$ is the projection of  $\Tcal$ to $\Ucal_{1,2}$ (the true class to which it belongs), and we see that $\Pcal_1$ only retains the characteristics of digit 1. 
Consequently, the test image looks the most similar to $\Pcal_1$, which also means that the underlying formation of the test image comes from the $\Ucal_{1,2}$ basis, and thus can be classified as digit 1. 

Numerically, using \Cref{def:frobenius}, the classification procedure makes the following calculation: $\|\Tcal - \Pcal_0\|_F\approx 0.0263 > \|\Tcal - \Pcal_1\|_F \approx 0.0089$. From this, we would categorize the test image as a digit 1. One can also make similar qualitative conclusions by visually observing $|\Tcal - \Pcal_0|$ and $|\Tcal - \Pcal_1|$. The bright pixels seen in $|\Tcal - \Pcal_0|$ illustrate stark differences in pixel values, whereas the more consistent dark coloring in $|\Tcal - \Pcal_1|$ indicates that the pixels are more similar.

\section{Numerical Experiments}
\label{sec:numericalExperiments}

Our numerical experiments aim to understand how our algorithm performs for our classification task under different choices of transformations and different truncations of the basis elements.
We include an overview of the choices of transformation $\bfM$ that we experiment with in \Cref{sec:parameters} and describe our application of the algorithm and its performance on fMRI data in \Cref{sec:starplus}. 
Related code can be found at~\url{https://github.com/elizabethnewman/tensor-fmri}. 

\subsection{Choices of $\mathbf{M}$} 
\label{sec:parameters}

The choice of $\starM$-product transformations can significantly impact the extracted features of the tensor. In our experiments, we implement three invertible transformation matrices to three dimensions when computing the $\starM$-product in t-SVDM. 
\begin{itemize}
    \item \textbf{Discrete Fourier Transform (t-product):} Based on the work in \cite{KilmerMartin2011}, we apply the one-dimensional fast Fourier transform \cite{CooleyJamesW} along a mode of $\Acal$. This specific choice of transformation can decompose signals to several separate frequencies and capture the sensitive shift of matrices in the spatial dimension. 
    \item \textbf{Discrete Cosine Transform (c-product):} Based on the work in~\cite{KernfeldKilmerAeron2015},  we apply the one-dimensional discrete cosine transform~\cite{fct} along a mode of $\Acal$. This is a more efficient implementation than explicitly forming the full discrete cosine transform matrix. 
    By transforming the signal from the spatial and temporal domain to the frequency domain, the DCT helps to separate data into parts of different importance. For this reason, the DCT is also often used in image compression domains \cite{dct-image}\cite{DBLP:journals/corr/RaidKEA14}.  
    \item \textbf{Haar Matrix:} Introduced in \cite{haar}, the Haar wavelet transformation is adept at dealing with abrupt transitions of signals \cite{haar_char}. Our implementation adopts the normalized version to ensure the matrix is orthogonal. One limitation of our implementation is that the dimensions of our matrix must be a power of two, i.e., $2^n \times 2^n$ for $n=1,2,\dots$.
    \item \textbf{Banded Matrix:} We apply the lower-triangular banded matrix defined in \cite{malik2020tensor}. This matrix was originally proposed for dynamic graphs, or graphs in which the nodes are fixed but edges and features can change in time. Since fMRI images are related at adjacent time points, the banded matrix could be an appropriate $\mathbf{M}$ for the temporal dimension of our data. 
    \item \textbf{Random Orthogonal Matrix:} We construct the random orthogonal matrix by retrieving the orthogonal matrix from the QR decomposition of a $n \times n$ matrix with entries sampled from a univariate Gaussian distribution of random floats with mean 0 and variance 1. No data structure is assumed when applying the random orthogonal matrix, giving it little advantage over other transformations that do assume and incorporate some kind of structure. However, we still choose to incorporate this matrix into our experiments and compare its performance to that of other choices of $\bfM$.
    \item \textbf{Data Dependent Matrix:} For a transformation along the $k$-th dimension, the data dependent matrix $\bfM_{k}$ is computed via the following:
        \begin{align*}\label{eq:dataDependentMatrix}
            \Acal_{(k)} = \bfU \bfSigma \bfV^\top \text{ and } \bfM_{k} = \bfU^\top
        \end{align*} 
    The advantage of this $\bfM_{k}$ is that it can capture structure specific to the data, despite not having the same physical features with other choices of $\mathbf{M}$. Such matrices are often used for efficient representations in transformations including Higher Order Singular Value Decomposition (HOSVD) \cite{KoldaBader2009:tensorBackground}. 
\end{itemize}

\subsection{fMRI Results}
\label{sec:starplus}
In this section, we provide our results when applying the t-SVDM classification procedure using various combinations of parameters to a classification task utilizing the StarPlus fMRI dataset. These results demonstrate the superiority of the t-SVDM method compared to the best possible equivalent matrix-based method, and also illustrate how incorporating knowledge about the data (such as time series information or annotated regions of interest) into our choice of transformation can impact performance.

\subsubsection{Data Setup}
\label{sec:Datasetup}

For our experiments, we use the StarPlus fMRI dataset \cite{starplus}, which is a publicly available dataset from Carnegie Mellon University's Center for Cognitive Brain Imaging. The StarPlus fMRI dataset is organized as follows: for a single human subject, 80 trials are completed where each trial corresponds to the subject either reading a sentence or viewing a picture. Each trial is composed of a series of fMRI scans  over a period of 16 time intervals spaced out over 500 milliseconds. The fMRI scan taken at each time point is three-dimensional, with 8 axial slices that are 64 by 64 pixels. For additional information about the conditions under which the trials were obtained, see \cite{starplus}. 
The three spatial dimensions, time dimension, and trials are concatenated to form a single five-dimensional tensor, which is visualized in \Cref{fig:data_visualization}. 

\begin{figure}[h]
    \centering
    \includegraphics[width=0.5\textwidth]{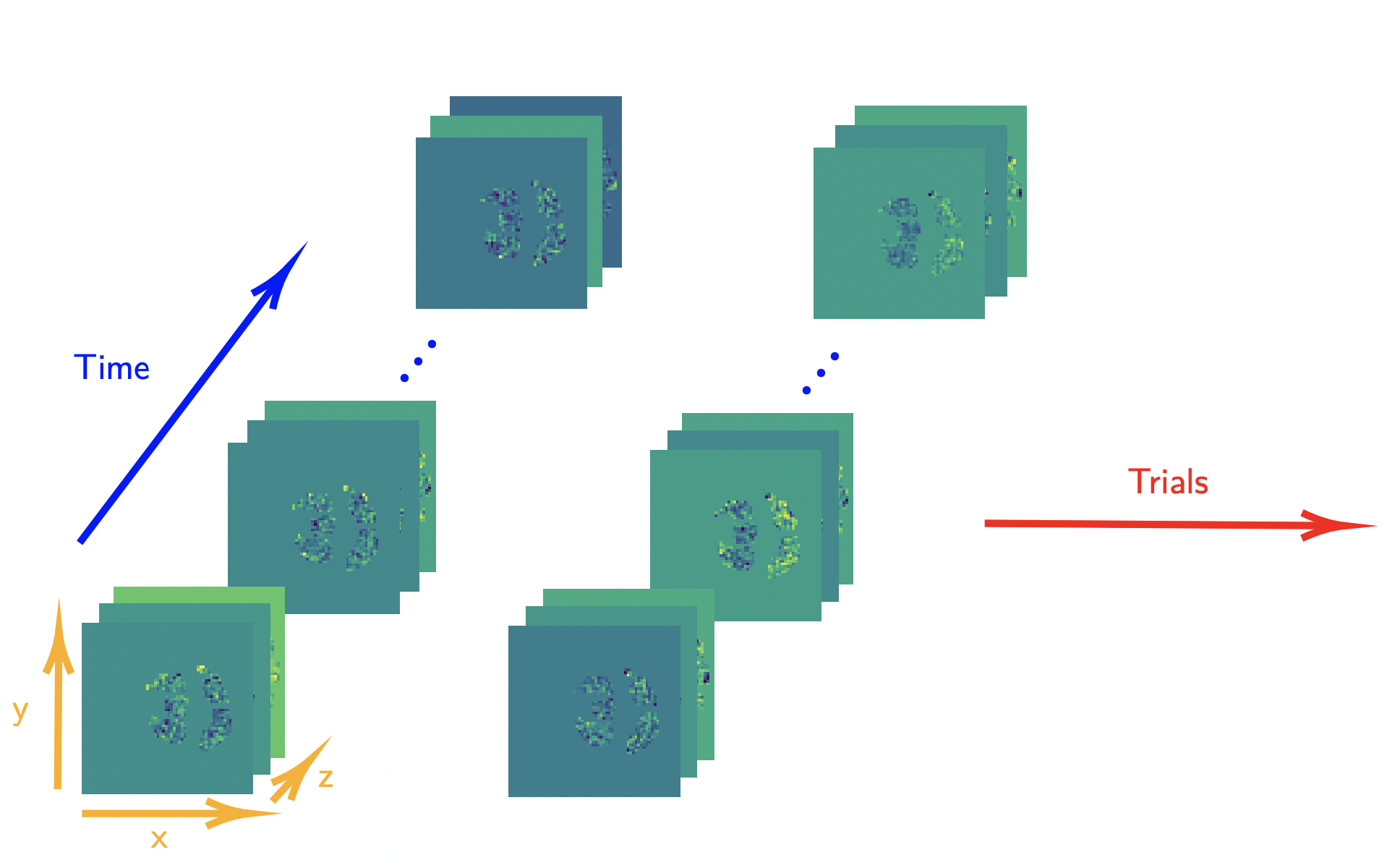}
    \caption{Visualization of data with dimensional shape (\textcolor[rgb]{1,0.68,0}{x}, \textcolor[rgb]{1,0.68,0}{y}, \textcolor[rgb]{1,0.68,0}{z}, \textcolor[rgb]{0,0.11,1}{time}, \textcolor[rgb]{1,0,0}{trials}) }
    \label{fig:data_visualization}
\end{figure}

We orient the tensor such that the the trials are indexed in the second dimension, giving us the following new dimensions: $(\text{x}, \text{trials}, \text{y}, \text{z}, \text{time}) = (64, 480, 64, 8, 16).$ Similar to how a sample of data is usually represented as a vector, we permute the tensor such that the second dimension contains trial information, allowing each trial to be stored as a lateral slice (analogous to vectors or columns) of the data tensor.

The comparable matrix-based approach would be to vectorize (\Cref{def:vectorize}) our high-dimensional data into a matrix. This is accomplished by unraveling all of the data corresponding to a single trial into one long column. These columns are then placed side-by-side to form a two-dimensional matrix where each column contains all of the spatial and time information for a single trial. 
The vectorized matrix shape is $(\textup{x} \times \textup{y} \times \textup{z} \times \textup{time}, \textup{trials}) = (64\times64\times8\times16, 480) = (524288,480)$ 

\subsubsection{Test Accuracy Results}
\label{sec:full_subject}
In this experiment, we use data from all six human subjects provided in the StarPlus dataset, resulting in a total of 480 trials. From these trials, we split the data such that 67\% of the trials are used for training and the remaining 33\% of the trials are used as test data. We divide the training trials based on labels so that we can construct two class tensors and compute local bases as described in  \Cref{sec:localTSVD}. Test data is stored as a tensor, and we project each lateral slice to local bases we produced on the last step. 
 
The matrix method would involve vectorizing the images as described earlier, computing the local SVD, and using the matrix version of projection and distance metric to make classification. We monitor the performance of our classification procedure by measuring how many test trials it was able to correctly classify. To compute this test accuracy, we calculate
\begin{align*}
    \text{test accuracy = }\frac{\text{number of correctly classified images}}{\text{number of images}}.
\end{align*}
\Cref{fig:prodMethod3} illustrates the relationship between the number of basis elements and the test accuracy for various $\bfM$. 
The main takeaways, as described below, are that our tensor-based method demonstrates superior performance over the traditional matrix-based approach; different transformation matrices provides different accuracies; the optimal truncation parameter also impacts results.

\paragraph{Tensor or Matrix} We know from \Cref{sec:tsvdm} that representing high-dimensional data as a tensor is provably optimal to its vectorized matrix form. Our results quantitatively illustrate that this optimal tensor representation also carries over into classification performance. 
    The tensor method outperforms the matrix method in terms of test accuracy in the following ways. First, with appropriate choice of transformation as will be described in the following point, the test accuracy is consistently higher than the matrix method for all choices of $k$. Second, we observe that tensor-based methods tend to start off well with a  limited number of basis elements, despite having the same storage cost for both matrix and tensor methods. Another computational benefit can be seen in \Cref{alg:tsvdm}, which shows that the t-SVDM offers potential for parallelization as well as only computing the SVD on relatively small frontal slice. 
\paragraph{Specific choices of $\bfM$} Selected choices of $\bfM$ are discussed in detail here to demonstrate how they each influence the test accuracy.
    \begin{enumerate}
        \item \textbf{Facewise Product:} The facewise product multiplies frontal slices of the data and is not designed to account for spatial and temporal change within the fMRI images. This suggests why its test accuracy is among the lowest. 
        
        \item \textbf{Haar-Banded:} When we use the banded matrix for the temporal dimension and the Haar matrix for all other transformations, we observe a consistently higher test accuracy than using the Haar transformation for all dimensions. This makes sense since there is a time relationship in fMRI data that we can exploit.
        This illustrates the advantages of choosing different $\bfM$ that are optimal for the nature of information being stored in a specific dimension.
        
        \item \textbf{Discrete Fourier Transform (t-product):} Since our fMRI data have similar slices among adjacent temporal dimensions as well as spatial dimensions, the t-product leads to the highest accuracy. 
    \end{enumerate}
\paragraph{Low vs. High Representation Power} The larger the number of basis elements, the more expressive our truncated basis $\Ucal_{i,k}$ will be. However, an overly expressive basis may represent classes equally well. Likewise, keeping too few basis elements could result in our basis $\Ucal_{i,k}$ failing to properly represent the most important features for a particular class. For example, as shown in \Cref{fig:prodMethod3}, some choices of transformations including Haar matrix and c-product demonstrate a trend of a brief increase followed by a decrease once the number of basis elements becomes too high.

\begin{figure}[h]
    \centering
    \includegraphics[width=0.8\linewidth, trim={0cm 0cm 0cm 2cm},clip]{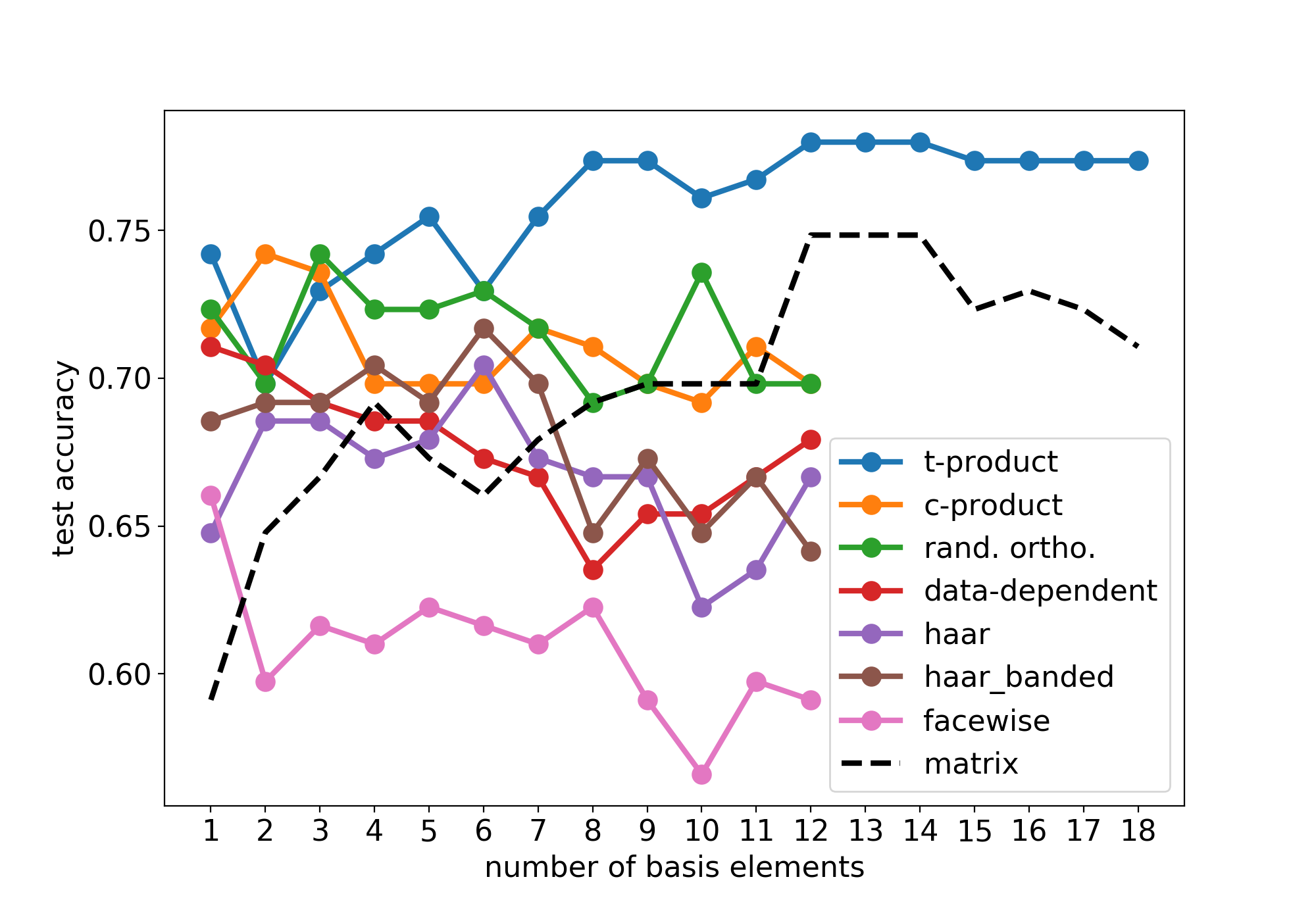}
    \caption{Test accuracy with respect to number of basis elements with varying $\bfM$. All methods are implemented for $k = 1,\dots,12$. We see a rise of the accuracy for the matrix method when $k$ approaches 12, so we extend $k$ to 18 for the matrix method. To give a complete comparison, the t-product is also run for $k = 1,\dots,18$. Most tensor-based methods implement the same $\bfM$ for all three transformed dimensions, except for the Haar-banded tensor method which implements banded matrix in the temporal dimension, and Haar matrix in the other two dimensions.}
    \label{fig:prodMethod3}
\end{figure}

\subsubsection{Utilizing Regions of Interest}
\label{sec:roi}

The StarPlus dataset is marked with 25-30 regions of interest, or ROIs, corresponding to anatomically defined regions of the brain. \Cref{fig:roi} shows an example of how the ROIs are provided for each image, with points on the colorbar associated with specific ROIs in the brain. For more information on the meanings of these abbreviated ROIs, please see \Cref{sec:roi_appendix}.

\begin{figure}[h]
    \centering
    \includegraphics[width=0.6\textwidth]{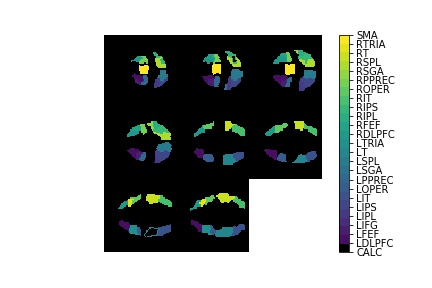}
    \caption{Regions of interest for 3D fMRI scan at single time point}
    \label{fig:roi}
\end{figure}

 To determine if incorporating knowledge about these anatomically-defined regions of interests impacts classification performance, we develop a transformation matrix that is formulated using the provided ROI data markups. To form this transformation, we use the following process. Let $\Acal\in \Rbb^{n_1\times n_2\times \dots \times n_p}$ and let $\Rcal \in \Rbb^{n_1 \times n_2\times \dots \times n_p}$. Each entry of $\Rcal$ is an integer that indicates the region of interest at that voxel. 

Let $m_k = n_1n_2\cdots n_{k-1}n_{k+1} \cdots n_p$ be the number of columns of the unfolded tensor $\Rcal_{(k)} \in \Rbb^{n_k\times m_k}$.  
Let $1 \le j_1 < j_2 < \dots < j_q \le m_k$ be the set of indices of columns of $\Rcal_{(k)}$ that contain a particular ROI label. 
Recall that since the columns of $\Rcal_{(k)}$ are vectorized mode-$k$ fibers, these columns contain fibers that cross through the desired region of interest. 
Form the ROI selection matrix $\bfP^{\rm ROI}_k \in \Rbb^{m_k\times q}$ that selects the columns of $\Rcal_{(k)}$ that contain the desired region of interest via matrix multiplication from the right; that is, $\Rcal_{(k)}\bfP^{\rm ROI}_k \in \Rbb^{m_k\times q}$.  
Here, each column of $\bfP^{\rm ROI}_k$ contains columns from an $m_k\times m_k$ identity matrix:
    \begin{align*}
        \bfP^{\rm ROI}_k(:,\ell) = \bfe_{j_\ell} \quad \text{for} \quad \ell=1,\dots, q.
    \end{align*}

To form the ROI data-dependent matrices, we take the following steps for $k=n_3,n_4, \dots, n_p$:
    \begin{align*}
        \Acal_{(k)}^{\rm ROI} &= \Acal_{(k)}\bfP^{\rm ROI}_k=\bfU\bfSigma\bfV^\top & \bfM_k^{\rm ROI} &= \bfU^\top.
    \end{align*}

We use the ROI-dependent transformation to transform the data into a domain that is created from the anatomical structure according to a specific brain region. We hypothesize that the ROIs from which a better-performing transformation is created might also be regions that are known to be associated with vision or language. We repeat each of these experiments for three human subjects. \Cref{fig:roi_results} displays results obtained when using an ROI-dependent $\bfM$.

\begin{figure}[H]
    \centering
    \begin{subfigure}{0.32\textwidth}
    \includegraphics[width=\textwidth]{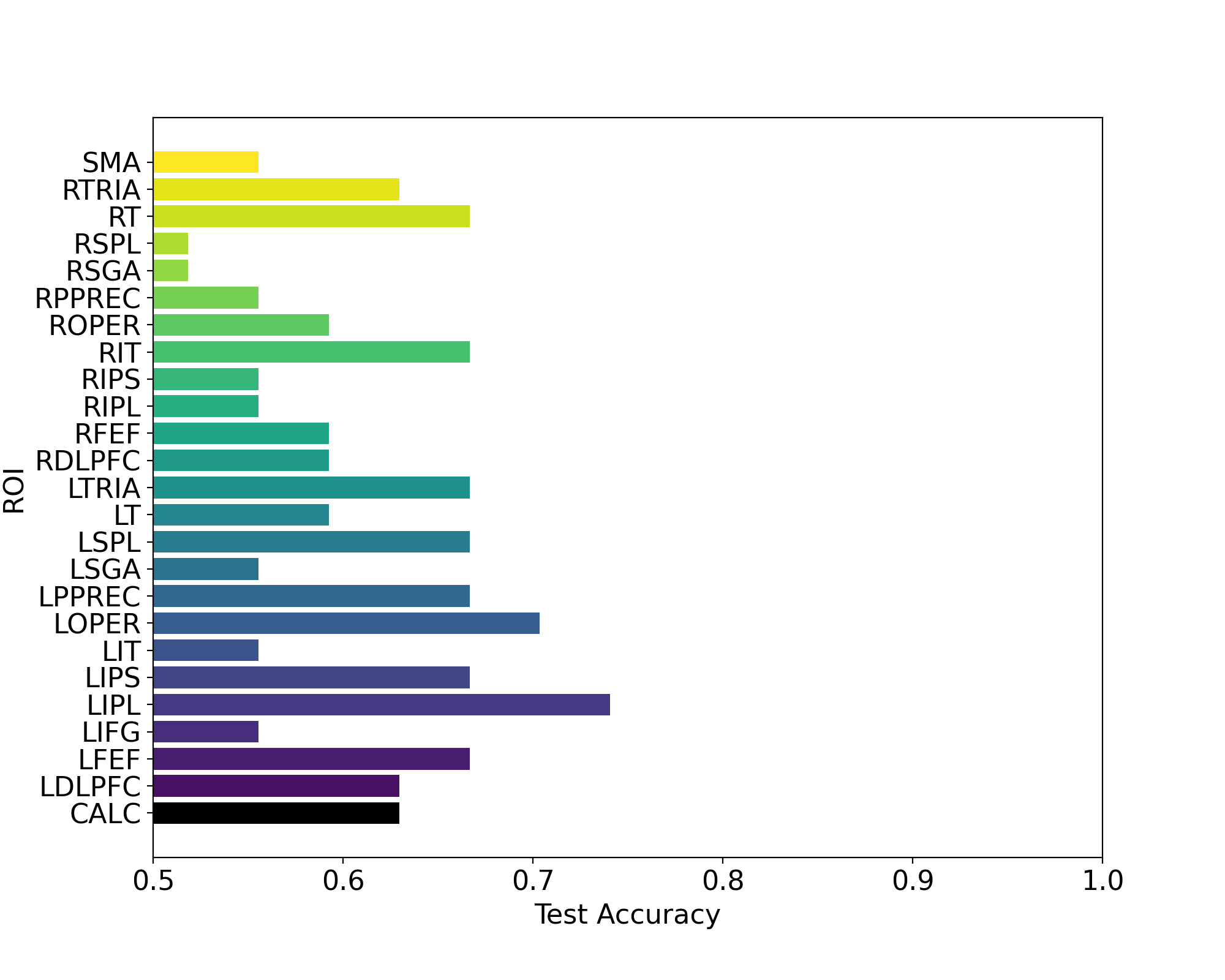}
    \caption{Subject 1}
    \end{subfigure}
    \hfill
    \begin{subfigure}{0.32\textwidth}
    \includegraphics[width=\textwidth]{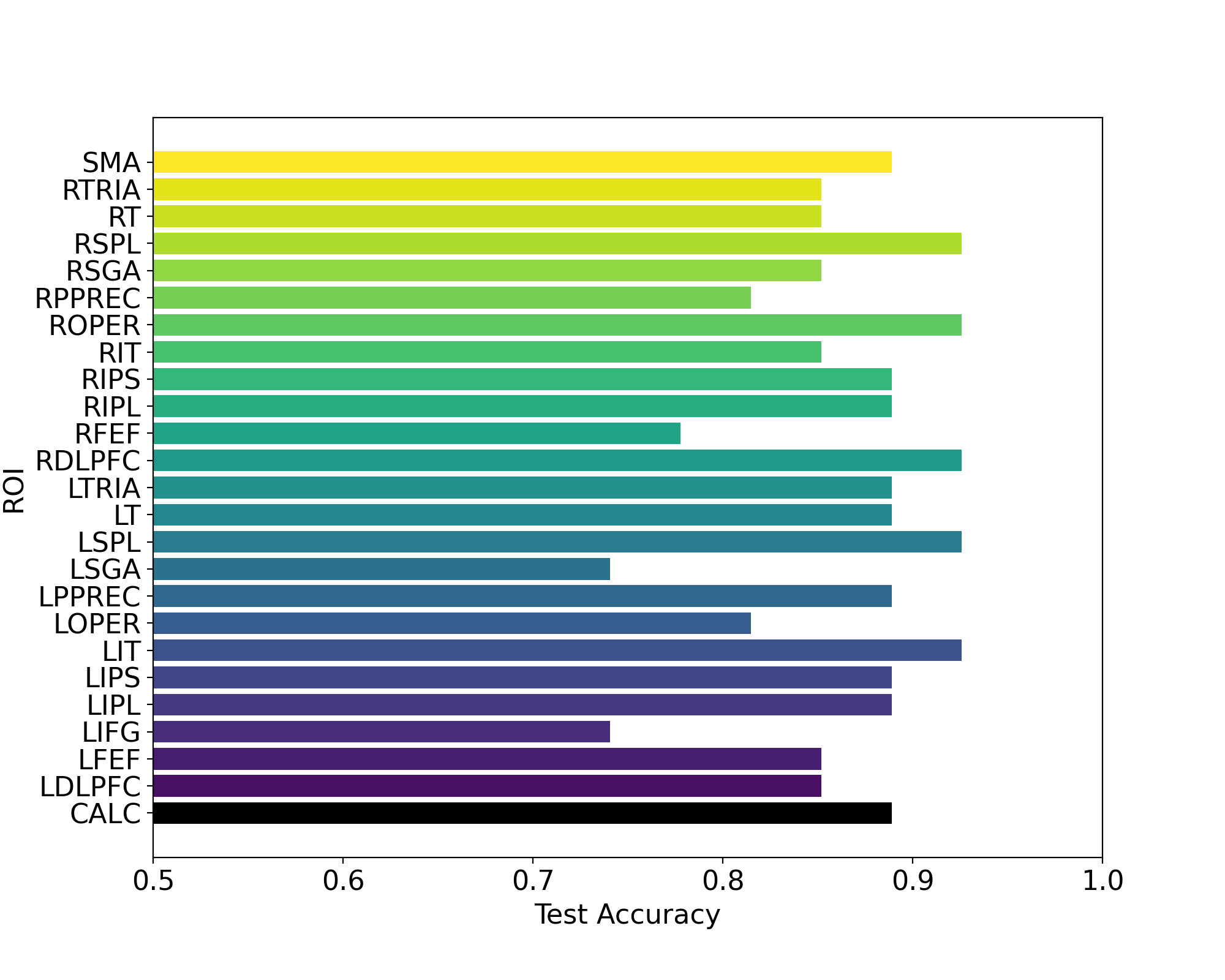}
    \caption{Subject 2}
    \end{subfigure}
    \hfill
    \begin{subfigure}{0.32\textwidth}
    \includegraphics[width=\textwidth]{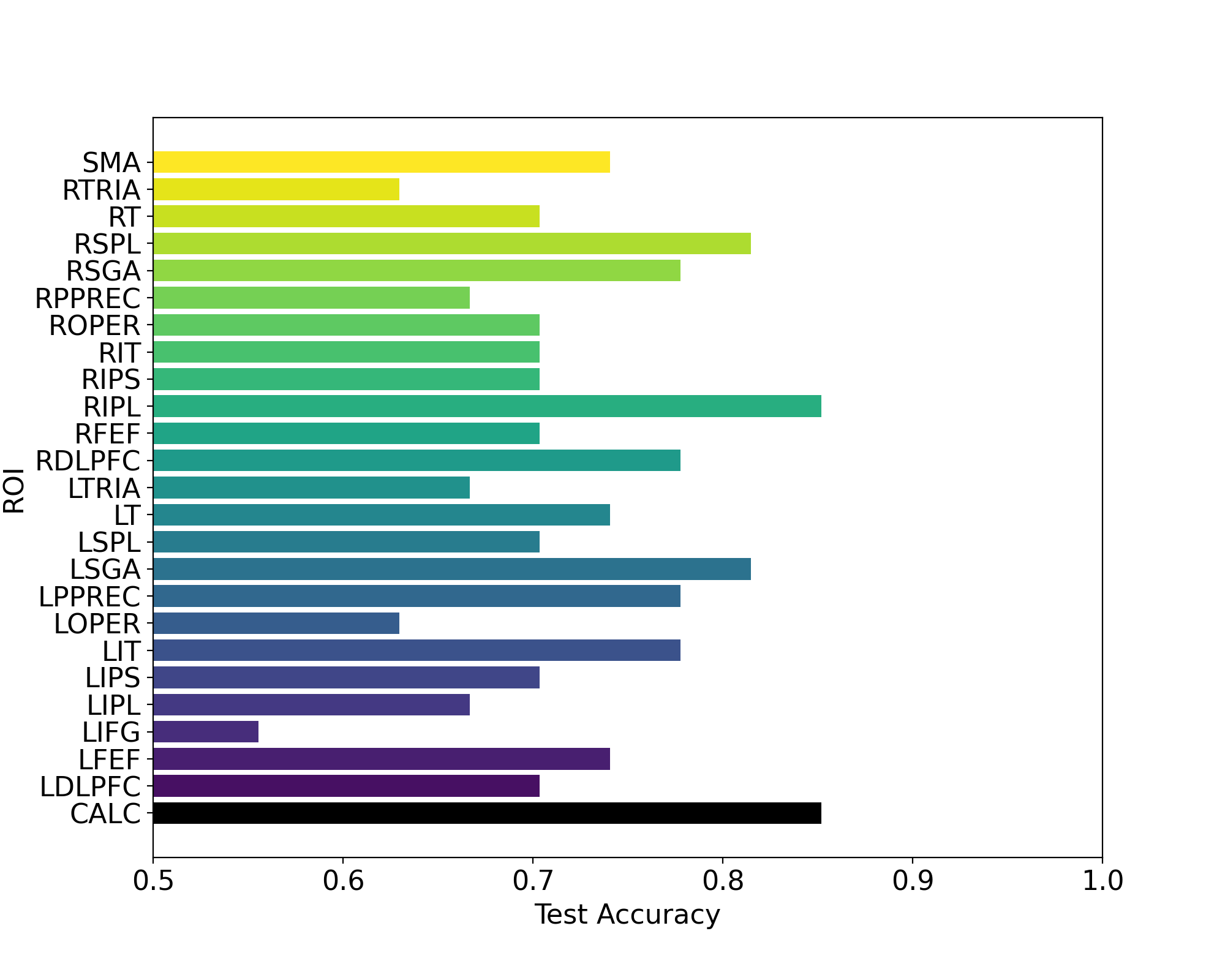}
    \caption{Subject 3}
    \end{subfigure}
    \caption{Results when using ROI-dependent $\bfM$ for $\bfM_{3}, \bfM_{4}$, and $\bfM_{5}$ with $k = 4$ for all dimensions}
    \label{fig:roi_results}
\end{figure}

From these results, we can see that there do exist choices of ROI's from which we can construct a $\bfM$, it is possible that offers improved classification accuracy over some of the choices of $\bfM$ illustrated in \Cref{fig:prodMethod3}. For example, for Subjects 2 and 3, there are multiple ROIs for which test accuracy exceeds 90\%, an accuracy which was never reached in any of our earlier experiments. However, it is also clear that there does not exist a specific ROI that consistently improves the accuracy across all subjects. Note that we only construct the ROI-dependent $\bfM$ from the data of a single subject. Therefore, it makes sense that the performance using that $\bfM$ will vary drastically from subject to subject.

Our interpretation of these results is that the regions that are most impactful for classification must be different depending on the human subject. Since individuals might cognitively experience the tasks of reading a sentence or viewing a picture differently, making generalizations about how all humans process these kinds of information is difficult. This observation gives insight into the challenges of constructing a t-SVDM approach that is sufficiently generalized to work for various brains but specific enough to produce accurate predictions. 

It is also possible that our dataset, while large in the sense that numerous trials are conducted for each subject, provides information from too few subjects. This would make it difficult to identify brain regions that would be universally impactful in the classification process. After all, six human subjects are hardly a representative sample of how all human brains work. There may actually be underlying universal similarities that could be detected by utilizing a dataset with more subjects. On the other hand, the variability introduced through the inclusion of more human subjects may make it more difficult for our multilinear approach to construct a good local basis.

Another potential culprit for the dramatic differences in accuracy among subjects could be some registration issues inherent in the StarPlus data. For example, if the MRI machine is oriented slightly differently the day that data was being collected for Subject 2 than it was for Subject 1, then our framework would not be able to fully execute the classification task to its best potential. This could perhaps be remedied by utilizing image registration techniques \cite{10.3389/fnins.2019.00909} to standardize our data.

We also recognize that our inability to construct an ROI-dependent transformation that produces consistently improved results across all subjects may simply represent an intrinsic limitation of our method. Brain activity is inherently nonlinear, produced through the discontinuous firing of neurons, and so the data that is obtained through monitoring this brain activity using fMRI would be representing nonlinear structure. Our methods in this study rely only on t-linear transformations and are not able to fully capture the nonlinear data structure. 

\section{Conclusions and Future Work}
\label{sec:conclusions}

Based on tensor notations and $\starM$-product, we extend the t-SVDM framework to $p$-dimensional tensors and use this to describe a local truncated t-SVDM approach for image classification, which can theoretically be applied to any high-dimensional labeled datasets. In our numerical experiments, we have been able to show that there does exist a t-SVDM approach that outperforms the best equivalent matrix-based SVD approach in terms of test accuracy, which quantitatively demonstrates the advantage of using tensor methods that preserve multilinear structure. Moreover, we find drastic differences in accuracy depending on the choices of transformation matrices, which encourages intentional product selections and requires understanding of the data. We also explore the implications of region of interests in the brain.
Our success could further the development of future tensor-based approaches for classification that are better able to accommodate the complexity of high-dimensional data.


We acknowledge that while we have been able to achieve success with applying the t-SVDM to the StarPlus dataset, there are some intrinsic mathematical limitations that may have inhibited its performance. 
First, we only extract multilinear features via the t-SVDM. While this is an improvement over extracting linear features via the matrix SVD, fMRI data may have more complex relationships (e.g., nonlinear) that our framework cannot easily exploit.
This could be remedied by combining our approach with a nonlinear classification method, such as neural networks. Second, our method is orientation-dependent: we treat the frontal slices differently than the other slices, hence the way we orient the data is crucial.
However, while other tensor-based approaches are orientation-independent, they do not have a natural analogy to projections like the $\starM$-framework does.
This motivates future methodological work of defining an orientation-independent approach in the $\starM$-framework (see \cite[Sec. 7]{kilmer2019tensortensor}) and a local bases classification approach for other tensor frameworks (e.g., Higher-Order SVD \cite{KoldaBader2009:tensorBackground} and Tensor-Train \cite{TensorTrainNeuralNetwork}) that are less dependent on orientation. Third, we focus on a binary classification task, only paying attention to the differences captured between how human brains perceive picture or sentences. Our proposed classification procedure is defined in \Cref{sec:localTSVD} for any number of classes, and so we are also interested in exploring how classification using the t-SVDM framework could be applied to tasks with more than two classes (e.g. more complex dataset described in \cite{RN8}).
Fourth, we hope to further develop these results to identify an approach that would be useful for medical diagnostic classification tasks. Using the local t-SVDM classification algorithm, we could analyze fMRI data for more significant medical challenges such as disease prediction and prevention.

\section*{Acknowledgments}
 This work was supported by the US National Science Foundation award DMS 2051019 and was completed during the ``Computational Mathematics for Data Science" REU/RET program in Summer 2021. We would like to thank Dr. Newman and the rest of the faculty supervising this REU/RET program for their guidance and input throughout this project.

\bibliographystyle{siamplain}
\bibliography{main.bib}

\appendix
\section{Appendix}
\begin{subsection}{Proofs for the Eckart-Young Theorem for Order-$p$ Tensors}
\label{sec:proofs_appendix}

\begin{lemma}[Tensor Orthogonal Invariance under $\starM$-product]
\label{lem:unitaryinvariance}
Let $\bfM_3 \in \mathbb{R}^{n_3 \times n_3 }, ..., \bfM_p \in \mathbb{R}^{n_p \times n_p }$ such that each $\bfM_{i}$ is an invertible non-zero scalar multiple of an orthogonal matrix $\bfW_{i}$ with scalars $c_i \in \mathbb{R}$ and $\starM$-orthogonal $\Qcal \in \mathbb{R}^{n \times n \times \dots \times n_{p}}$ . 
Then, for  $\Acal \in \mathbb{R}^{n \times \ell \times \dots \times n_{p}}$, we have $\| \Qcal \starM \Acal \|_{F} = c \| \Acal \|_F $, with $c \in \mathbb{R}$. 
Likewise, if $\Acal \in \mathbb{R}^{\ell \times n \times \dots \times n_{p}}$, $\| \Acal \starM \Qcal \|_{F} = c \| \Acal \|_F $.
\end{lemma}

\begin{proof}
Let $\Acal \in \mathbb{R}^{n \times \ell \times \dots \times n_{p}}$. 
First, we show that the norm of $\Acal$ is preserved (up to scalar multiplication) in the transform domain.
From \cite{kilmer2019tensortensor}, we know that $ \| \Acal \times_{3} \bfM_{3} \|_F =  c_{3} \| \Acal \|_F $.
This generalizes to the mode-$k$ product as follows.  
Assume $\bfM_k = c_k\bfW_k$ where $\bfW_k \in \Rbb^{n_k\times n_k}$ is orthogonal. 
Then,
    \begin{align*}
        \|\Acal \times_{k} \bfM_{k}\|_F 
            = \| \bfM_k\Acal_{(k)} \|_F = \| c_k\bfW_k\Acal_{(k)} \|_F=c_k\|\Acal_{(k)}\|_F = c_k\|\Acal\|_F.
    \end{align*}
If we apply multiple transformation matrices to each dimension from 3 to $p$, we obtain the following:
    \begin{align*}
        \| \hat{\Acal} \|_F &= \| \Acal \times_3 \bfM_3 \dots \times_p \bfM_p \|_F \\
        &= \| \bfM_{p} (  \Acal \times_3 \bfM_3 \dots \times_{p-1} \bfM_{p-1} )_{(p)} \|_F \\
        &=\| c_{p} \bfW_{p} (\Acal \times_3 \bfM_3 \dots \times_{p-1} \bfM_{p-1})_{(p)} \|_F\\
        &=c_p\|\Acal \times_3 \bfM_3 \dots \times_{p-1} \bfM_{p-1}\|_F\\
        &\vdots\\
        &=c\|\Acal\|_F\quad \text{where }c = c_{3} c_{4} \cdots c_p.
    \end{align*}

We now show the norm-invariance of the $\starM$-product when multiplying by an orthogonal tensor.
Let $\Qcal$ be orthogonal and $\Ccal = \Qcal \starM \Acal$. Then,
\begin{align*}
    \| \Acal \|_{F}^{2} &= \frac{1}{c^2 } \| \hat{\Acal} \|_{F}^{2} \\
    &= \frac{1}{c^2 } \sum_{i_3 = 1}^{n_3} \dots \sum_{i_p = 1}^{n_p} \| \Acal_{:,:,i_{3},...,i_{p} } \|_F^2\\
    &= \frac{1}{c^2 }  \sum_{i_3 = 1}^{n_3} \dots \sum_{i_p = 1}^{n_p} \| \hat{\Qcal}_{:,:,i_{3},...,i_{p} } \hat{\Acal}_{:,:,i_{3},...,i_{p} } \|_F^2  \textup{ since frontal slices of $\hat{\Qcal}$ are orthogonal} \\
    &=  \frac{1}{c^2 } \| \hat{\Ccal} \|_F \\
    &=  \| \Ccal \|_F .
\end{align*}
The other direction is similar. 
\end{proof}

\begin{lemma}[Ordering of Singular Tubes] 
\label{lem:orderSingularTubes}
Given the t-SVDM of $\Acal$ where $\Acal$ is of $t$-rank-$r$, we have $\| \Acal \|^{2}_F = \| \Scal \|^{2}_F = \sum_{i=1}^{r} \| \bfs_{i} \|^{2}_{F}$ where $\bfs_i$ is the $(i,i)$-tube of $\bfS$.  
Moreover, $\|\bfs_{1} \|_{F} \geq \|\bfs_{2} \|_{F} \geq \cdots \geq \|\bfs_{r} \|_{F} > 0$. 
\end{lemma}
\begin{proof}
We know that $\Ucal$ and $\Vcal^\top$ are $\starM$-orthogonal and the only entries of $\Scal$ lie along the diagonal of its frontal slices. By results shown in \cite{kilmer2019tensortensor}, we have
$$
\| \Acal \|^{2}_{F} = \| \Ucal \Scal \Vcal^\top \|^{2}_{F} = \| \Scal \|^{2}_{F} = \sum_{i=1}^{r} \| \bfs_{i} \|^{2}_{F} .
$$
We now prove the second part. From \Cref{lem:unitaryinvariance}, we have
$$
\| \bfs_{i} \|_{F}^2 = \frac{1}{c^2} \| \hat{\bf{s}}_{i} \|_{F}^2 =  \frac{1}{c^2} \sum_{i_3 = 1}^{n_3} \dots \sum_{i_p = 1}^{n_p}  \hat{\sigma}_{i,i,i_3 ,..., i_p}^{2}
$$
where $\hat{\sigma}_{i}^{(i_3 ,..., i_p ) }$ is the $i$-th largest singular value in the transform domain located at the frontal slice  along fixed indices $i_3 , ... , i_p $. By the definition of the matrix SVD, we know that $\hat{\sigma}_{i}^{(i_3 ,..., i_p )} \geq \hat{\sigma}_{i+1}^{(i_3 ,..., i_p ) },$ or that the singular values lie along the diagonal of the singular value matrix in descending order of magnitude. 
Thus, we have
$$
\| \bfs_{i} \|_F \geq \| \bfs_{i+1} \|_{F}.
$$
Hence, the norm of the singular tubes is ordered.
\end{proof}

The proof of the Eckart-Young-like \Cref{thm:eckartYoung} is provided below.
\begin{proof}
We first obtain a formulation for the error. We know that for the matrix SVD, we have $\|\bfA - \bfA_{k} \|_{F}^2 = \sum_{i= k+1}^{r} \sigma_{i}^2$ \cite{strang}. Thus, we have 
\begin{align*}
    \| \Acal - \Acal_{k} \|^{2}_{F} &= \frac{1}{c^2} \| \hat{\Acal} - \hat{\Acal}_{k} \|^{2}_{F} \\
    &=  \frac{1}{c^2} \sum_{{i_3} = 1}^{{n_3}} \dots \sum_{{i_p} = 1}^{n_p} \| \hat{\Acal}_{:,:,{i_3} ,..., {i_p} } - \hat{\Acal}_{:,1:k,{i_3} ,..., {i_p} }  \|^{2}_{F} \\
     &=  \frac{1}{c^2} \sum_{{i_3} = 1}^{n_3} \dots \sum_{{i_p} = 1}^{{n_p}} \sum_{i= k+1}^{r} \hat{\sigma}_{i,i,{i_3} ,..., {i_p} }^2 \quad \textup{(\Cref{lem:orderSingularTubes})} \\
    &= \frac{1}{c^2} \sum_{i= k+1}^{r} \| \bfs_{i} \|^{2}_{F}.
\end{align*}
Let $\Bcal$ be a tensor of t-rank-$k$. Then, using the fact that the SVD produces the best possible rank-$k$ approximation of a matrix, we have
\begin{align*}
    \| \Acal - \Bcal \|^{2}_{F} &= \frac{1}{c^2} \| \hat{\Acal} - \hat{\Bcal} \|^{2}_{F} \quad \textup{(\Cref{lem:unitaryinvariance})}\\
    &= \frac{1}{c^2} \| \reshape(\hat{\Acal},[n_1, n_2, n_3n_4\cdots n_p] ) - \reshape (\hat{\Bcal},[n_1, n_2, n_3n_4\cdots n_p]) \|^{2}_{F}\\
    &\geq \frac{1}{c^2} \| \reshape (\hat{\Acal},[n_1, n_2, n_3n_4\cdots n_p]) - \reshape(\hat{\Acal}_{k},[n_1, n_2, n_3n_4\cdots n_p]) \|^{2}_{F} \\
    &=  \| \Acal - \Acal_{k} \|^{2}_{F} .
\end{align*}

When we compute $\reshape(\hat{\Bcal},[n_1, n_2, n_3n_4\cdots n_p])$, we are concatenating the decoupled frontal slices in the transform domain into a third-order tensor.  
Thus, we can use the results from~\cite{kilmer2019tensortensor} and immediately obtain the inequality above.


\end{proof}
\end{subsection}

\begin{subsection}{Additional Descriptions for Regions of Interest}
\label{sec:roi_appendix}
To clarify some of the abbreviated ROIs referenced in \Cref{sec:roi}, we provide the full names of each of the ROIs in the below table.

\begin{center}
    \begin{tabular}{|l|l|}
    \hline
    Abbreviation & Name \\
    \hline
    \hline
    LTRIA/RTRIA & Left/Right Triangularis \\
    \hline
    LIT/RIT & Left/Right Inferior Temporal Lobe \\
    \hline 
    LIPS/RIPS & Left/Right Intraparietal Sulcus \\
    \hline
    LSGA/RSGA & Left/Right Supramarginal Gyrus \\
    \hline 
    LT/RT & Left/Right Temporal Lobe \\
    \hline
    LOPER/ROBER & Left/Right Opercularis \\
    \hline
    LSPL/RSPL & Left Superior Parietal Lobe \\
    \hline
    LPPREC/RPPREC & Left/Right Posterior Precentral Sulcus \\
    \hline
    LDLPFC/RDLPFC & Left/Right Dorsolateral Prefrontal Cortex \\
    \hline 
    CALC & Calcarine Sulcus\\
    \hline
    
    \end{tabular}
\end{center}

The above terms, along with more detailed explanations of each of these regions, can be found in \cite{roi}. Our initial hypothesis was that regions known to be associated with vision and language (e.g. the left temporal lobe, which is associated with language skills) might have a greater role in improving accuracy when utilized as an ROI-dependent $\bfM$ as described in \Cref{sec:roi}. 
\end{subsection}

\end{document}

%% file: macros.tex
\usepackage{amsmath} 
\usepackage{tikz} 
\usepackage{graphicx} 
\usepackage{stackengine}
\usepackage{xcolor} 
\usepackage{multicol} 
\usepackage{hyperref} 
\hypersetup{colorlinks=true, urlcolor=cyan,}
\usepackage{algorithm, algorithmic} 
\usepackage{multirow}
\usepackage{subcaption}

\DeclareMathOperator*{\argmin}{arg\, min}

\DeclareMathOperator*{\col}{col}
\DeclareMathOperator*{\reshape}{reshape}
\DeclareMathOperator*{\myVec}{vec}
\DeclareMathOperator*{\trank}{t-rank}

\DeclareMathOperator*{\st}{s.t.}

\def\mydefb#1{\expandafter\def\csname bf#1\endcsname{\mathbf{#1}}}
\def\mydefallb#1{\ifx#1\mydefallb\else\mydefb#1\expandafter\mydefallb\fi}
\mydefallb aAbBcCdDeEfFgGhHiIjJkKlLmMnNoOpPqQrRsStTuUvVwWxXyYzZ\mydefallb

\def\mydefb#1{\expandafter\def\csname #1bb\endcsname{\mathbb{#1}}}
\def\mydefallb#1{\ifx#1\mydefallb\else\mydefb#1\expandafter\mydefallb\fi}
\mydefallb aAbBcCdDeEfFgGhHiIjJkKlLmMnNoOpPqQrRsStTuUvVwWxXyYzZ\mydefallb

\def\mydefb#1{\expandafter\def\csname #1cal\endcsname{\boldsymbol{\mathcal{#1}}}}
\def\mydefallb#1{\ifx#1\mydefallb\else\mydefb#1\expandafter\mydefallb\fi}
\mydefallb aAbBcCdDeEfFgGhHiIjJkKlLmMnNoOpPqQrRsStTuUvVwWxXyYzZ\mydefallb

\newcommand{\bfSigma}{{\boldsymbol \Sigma}}

\newcommand{\starM}{{\star_{\rm M}}}

